\documentclass[11pt]{article}

\usepackage[final]{acl}

\usepackage{times}
\usepackage{latexsym}

\usepackage[T1]{fontenc}
\usepackage[utf8]{inputenc}

\usepackage{microtype}
\usepackage{inconsolata}
\usepackage{graphicx}

\usepackage{booktabs}
\usepackage{amsmath}
\usepackage{amssymb}
\usepackage{amsthm}
\usepackage{array}
\usepackage{multirow}
\usepackage{makecell}
\usepackage{tabularx}
\usepackage{enumitem}
\usepackage{xspace}
\usepackage{placeins}
\usepackage{algorithm}
\usepackage{algpseudocode}
\algrenewcommand\algorithmicrequire{\textbf{Input:}}
\algrenewcommand\algorithmicensure{\textbf{Output:}}

\title{Target-Aware Calibration Data Selection for Preserving Uncertainty in Quantized Language Models}

\author{
Zhen Yang$^{1,\dagger}$,
Sizai Hou$^{2,\dagger}$,
Kaiwen Zheng$^{3}$,
Yaofang Liu$^{4}$,
Liang He$^{5}$
\\[-0.1em]
{\bfseries
Yixuan Chen$^{6,*}$,
Kangning Cui$^{7,*}$}
\\[0.2em]
\normalfont
\mbox{$^{1}$Yale University}
\qquad
\mbox{$^{2}$The Hong Kong University of Science and Technology}
\\
\mbox{$^{3}$The Hong Kong University of Science and Technology (Guangzhou)}
\\
\mbox{$^{4}$City University of Hong Kong}
\qquad
\mbox{$^{5}$Shanghai Institute of Optics and Fine Mechanics}
\\
\mbox{$^{6}$University of Oxford}
\qquad
\mbox{$^{7}$City University of Hong Kong (Dongguan)}
\\
$^\dagger$ Equal contribution
\qquad
$^*$ Corresponding authors
}

\newcommand{\dpq}{\textsc{DPQ}\xspace}
\newcommand{\fp}{\textsc{FP}\xspace}
\newcommand{\gptq}{\textsc{GPTQ}\xspace}
\newcommand{\bnb}{\textsc{BNB-NF4}\xspace}
\newcommand{\squad}{\textsc{SQuAD2}\xspace}
\newcommand{\jsd}{\mathrm{JSD}}
\newcommand{\mcqa}{\textsc{MCQA}\xspace}
\newcommand{\calD}{\mathcal{D}}
\newcommand{\calY}{\mathcal{Y}}
\newcommand{\R}{\mathbb{R}}
\newtheorem{proposition}{Proposition}
\newtheorem{corollary}{Corollary}
\usepackage[table]{xcolor}
\definecolor{bestblue}{HTML}{C9DAF8}    
\definecolor{secondblue}{HTML}{EAF2F8}  
\definecolor{headergray}{HTML}{E6E6E6}
\definecolor{rowgray}{HTML}{F7F7F7}

\begin{document}
\maketitle

\begin{abstract}
Quantization is widely used to deploy large language models, but its effect on uncertainty behavior, such as confidence, margins, and abstention, is rarely treated as a primary objective. We frame calibration-data selection for quantization as a \emph{target-dependent uncertainty-preservation problem}. Different deployments emphasize different regions of the input distribution, yet prior work mainly optimizes accuracy-oriented compression metrics or adjusts scores after quantization. We formalize this goal with distributional and boundary preservation risks, and provide a simple mixture-mismatch argument explaining why no single calibration recipe should be expected to fit all targets. We introduce Doubt-Preserving Quantization (\dpq), a lightweight pre-quantization recipe family that uses full-precision predictions to construct target-aligned calibration mixtures of high-doubt examples and generic anchors. Across 8 language models, 9 NLP benchmarks, and 22 comparison methods, the leading fixed recipe changes with the preservation target: \dpq-r75 leads on \squad answerability-boundary preservation, while milder or single-signal variants, including \dpq-r50, confidence-only, and entropy-only, better preserve broad multiple-choice QA behavior. These results show that calibration data should be selected for the specific full-precision score behavior a deployment needs to preserve, rather than treated as a fixed quantization detail. {Code is available at \url{https://github.com/xi-xiaoran/DPQ}.}
\end{abstract}

\section{Introduction}

Quantization has become a standard tool for deploying large language models (LLMs) under memory and compute budgets. Weight-only and low-bit post-training methods such as \gptq \citep{frantar2023gptq}, \textsc{AWQ} \citep{lin2023awq}, SmoothQuant \citep{xiao2023smoothquant}, OmniQuant \citep{shao2024omniquant}, \textsc{LLM.int8()} \citep{dettmers2022llmint8}, SpQR \citep{dettmers2024spqr}, and QLoRA-style 4-bit representations \citep{dettmers2023qlora} make large models accessible on smaller hardware. Most compression evaluations, however, still focus on accuracy, perplexity, or benchmark scores~\citep{he2026budgetdraft,wu2026flashsvd}. This is incomplete for applications that rely on model scores for abstention, reranking, ensembling, verification, or answerability decisions. In these settings, preserving confidence, margins, and answerability decisions can be as important as preserving the top-1 answer~\citep{shao2026decodeshare}. 
Recent work shows that quantization can shift confidence \citep{proskurina-etal-2024-quantization} and that calibration data affects compression quality \citep{williams-aletras-2024-impact}, but calibration-data selection rarely targets preservation of \fp uncertainty behavior.

\begin{figure}[t]
    \centering
    \includegraphics[width=\linewidth]{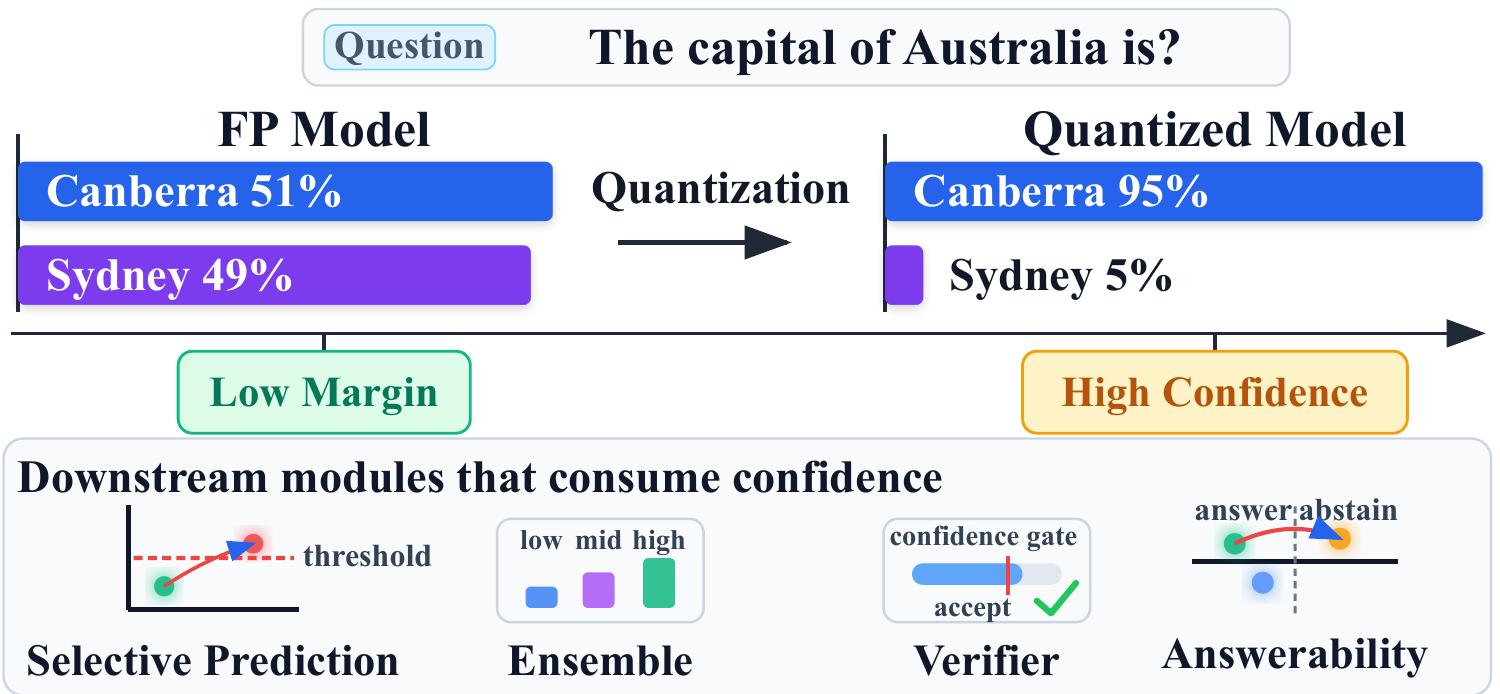}
    \caption{
    Motivating example of uncertainty drift after quantization.
    The top-1 answer is unchanged, but the confidence margin shifts substantially, affecting downstream uncertainty-aware decisions.
    }
    \label{fig:motivation}
\end{figure}

This distinction matters even when accuracy is unchanged. As illustrated in Figure~\ref{fig:motivation}, an \fp model may assign option probabilities $(0.51,0.49)$ while its quantized counterpart assigns $(0.95,0.05)$, preserving the same top-1 answer but changing the confidence margin. Such shifts can alter downstream uncertainty-aware decisions, especially in answerability detection, where a small score change may determine whether the system answers an unanswerable question. Existing approaches do not directly target this preservation problem. Accuracy-oriented calibration-data studies \citep{williams-aletras-2024-impact} optimize reconstruction quality or label accuracy, while post-hoc calibration \citep{guo2017calibration,zhong-etal-2025-quantized} adjusts scores after quantization. Temperature-style maps preserve the argmax and cannot repair decision flips, whereas more flexible maps can improve label calibration at the cost of moving the quantized model away from \fp behavior.

We therefore frame calibration-data selection for quantization as a target-dependent \emph{uncertainty-preservation problem}. Broad answerable tasks require distributional preservation of option probabilities, confidence, and margins, whereas answerability or abstention-heavy tasks require boundary preservation of low-margin answer/no-answer decisions. We formalize these targets as two preservation risks and give a mixture-mismatch argument that explains why no fixed mixture of boundary examples and generic anchors should be expected to fit all targets unless the targets coincide. Building on this result, we introduce \dpq (Doubt-Preserving Quantization), a lightweight pre-quantization framework that selects high-doubt and answerability-boundary examples using \fp predictions, mixes them with generic anchors, and runs the unchanged quantizer. Across 8 language models, 9 NLP benchmarks, and 22 comparison methods, we observe this target dependence: \dpq-s128-r75 performs best for \squad answerability-boundary preservation, while broad \mcqa favors milder or single-signal recipes such as \dpq-r50, confidence-only, and entropy-only. Post-hoc and \textsc{AWQ} analyses further show that preservation signals are partly transferable, but score-based metrics remain quantizer-specific.

\paragraph{Contributions.}
\begin{itemize}[leftmargin=*, itemsep=0pt, topsep=0pt]
\item \textbf{Problem.} We formulate calibration-data selection for quantization as preserving full-precision uncertainty behavior, with distributional and boundary risks for different deployment targets.
\item \textbf{Method.} We provide a mixture-mismatch argument and introduce \dpq as a lightweight target-aware recipe family: it uses \fp predictions to select high-doubt or boundary-near examples, mixes them with generic anchors, and leaves the quantizer unchanged.
\item \textbf{Evaluation.} Across 8 language models, 9 NLP benchmarks, and 22 comparison methods, we show that the best recipe depends on the preservation target: high-boundary mixtures better preserve answerability boundaries, while milder recipes better preserve broad \mcqa behavior.
\end{itemize}

\section{Related Work}
\label{sec:related}

\paragraph{LLM quantization.}
Post-training LLM quantization includes approximate second-order weight quantization \citep{frantar2023gptq}, activation-aware and activation-smoothing methods \citep{lin2023awq,xiao2023smoothquant}, 8-bit and 4-bit representations \citep{dettmers2022llmint8,dettmers2023qlora}, and recent low-bit PTQ systems \citep{yao2022zeroquant,shao2024omniquant,dettmers2024spqr,kim2024squeezellm,egiazarian2024aqlm,tseng2024quipsharp,ashkboos2024quarot}. These methods primarily optimize reconstruction or downstream accuracy. Calibration data supports quantization, but is less often selected with the explicit goal of preserving confidence, margins, or abstention behavior.

\paragraph{Calibration data and selection.}
\citet{williams-aletras-2024-impact} show that calibration data affects pruning and quantization, while self-calibration uses the model itself to generate calibration data for quantization and pruning \citep{williams2025self}. We study a different objective: selecting calibration data to preserve full-precision uncertainty behavior. Related data-selection methods, including uncertainty sampling, core-set selection, and gradient- or representation-diversity selection, have been studied for data efficiency \citep{settles2009active,sener2018active,ash2020badge}. These objectives are useful controls, but they do not directly target quantized-vs-\fp option-probability preservation; we therefore include them as baselines.

\paragraph{Confidence calibration.}
Calibration has a long history in probabilistic prediction and evaluation \citep{brier1950verification,niculescu2005predicting,guo2017calibration,nixon2019measuring,platt1999probabilistic,zadrozny2002transforming,naeini2015obtaining,kull2019dirichlet,kumar2019verified,minderer2021revisiting}. Selective prediction and uncertainty estimation have also been widely studied \citep{chow1970optimum,elyaniv2010foundations,geifman2017selective,ovadia2019trust}. For language models, prior work studies confidence, self-knowledge, calibration, and ``know-when-you-do-not-know'' behavior in QA, prompting, and tuning settings \citep{jiang-etal-2020-know,zhao2021calibrate,desai-durrett-2020-calibration,xiao-etal-2022-uncertainty,kadavath2022language,kapoor2024large,kapoor2024calibration}. These works mainly analyze \fp models, prompting, or direct tuning, rather than calibration-data choice during quantization.

\paragraph{Post-hoc calibration.}
\citet{proskurina-etal-2024-quantization} document confidence shifts under low-bit compression, and \citet{zhong-etal-2025-quantized} propose soft-prompt post-hoc calibration. These works intervene after quantization is fixed. Temperature-style maps preserve the argmax and cannot repair decision or answerability flips; flexible score-space calibrators can change decisions but may move the model away from \fp behavior \citep{guo2017calibration,zadrozny2002transforming,naeini2015obtaining,kull2019dirichlet}. Our work instead studies calibration-data selection before quantization.

\section{Problem Statement}
\label{sec:setup}

We formulate calibration-data selection for quantization as uncertainty preservation with a task-dependent target. \S\ref{sec:risks} defines distributional and boundary risks. \S\ref{sec:boundary_fragility} explains why low-margin examples are fragile under quantization perturbations. \S\ref{sec:no_universal} gives a mixture-mismatch argument that motivates the \dpq design space in \S\ref{sec:method}.

\subsection{Preservation Risks}
\label{sec:risks}

\paragraph{Option scoring.}
For an input $x$ with candidate options $\calY(x)=\{y_1,\ldots,y_K\}$, a full-precision model assigns a score $s_0(x,y_i)$ to each option and induces
\begin{equation}
 p_0(y_i\mid x)=\frac{\exp s_0(x,y_i)}{\sum_{j=1}^K \exp s_0(x,y_j)},
 \label{eq:fp_prob}
\end{equation}
where $s_0$ is computed from conditional language-model likelihood. A quantized model produced with calibration set $\calD$ induces scores $s^Q_{\calD}(x,y_i)$ and probabilities $p^Q_{\calD}(y_i\mid x)$ analogously. Standard accuracy checks whether the top option matches the label; we instead ask whether $p^Q_{\calD}$ preserves the \emph{uncertainty behavior} of $p_0$. In \squad, answerability is scored as a two-option choice, and $p(\mathrm{answerable})$ denotes the softmax probability of the answerable option.

\paragraph{Risks.}
Expectations are over the target evaluation distribution. We define two preservation risks. The first is
\begin{equation}
  \mathcal{R}_{\mathrm{dist}}(\calD)=\mathbb{E}_{x}\big[\jsd(p^Q_{\calD}(\cdot\mid x),p_0(\cdot\mid x))\big],
\end{equation}
which measures quantization-induced change in the option distribution. The second is
\begin{equation}
  \mathcal{R}_{\mathrm{bdry}}(\calD)=\mathbb{E}_{x}\big[\mathbf{1}\{a^Q_{\calD}(x)\neq a_0(x)\}\,w(x)\big],
\end{equation}
where $a_0(x)=\arg\max_y p_0(y\mid x)$ and $a^Q_{\calD}(x)=\arg\max_y p^Q_{\calD}(y\mid x)$ are the \fp and quantized decisions, and $w(x)\geq 0$ upweights answerability or low-margin cases. $\mathcal{R}_{\mathrm{dist}}$ matters when the quantized model should act as a drop-in replacement for an \fp model, as in broad \mcqa, reranking, or ensembling. $\mathcal{R}_{\mathrm{bdry}}$ matters when answerability or abstention decisions are central, as in \squad answerability, safety filters, or selective prediction.

\paragraph{Metrics.}
Empirically, we instantiate these risks with ECE \citep{guo2017calibration}, adaptive ECE \citep{nixon2019measuring}, NLL, and Brier score \citep{brier1950verification}, together with \fp-behavior metrics: \fp agreement, $\jsd(p^Q_{\calD},p_0)$, confidence shift, top-two margin drift, and margin correlation. For \squad, we also report boundary-accuracy deviation and answerability-rate deviation, which measure how far the quantized answer/abstain decisions and answer rate move from \fp. These metrics evaluate fidelity to the full-precision reference rather than improved ground-truth calibration: \fp agreement and JSD-to-\fp measure whether a quantized model can replace an already validated \fp model in downstream score-consuming pipelines.

\subsection{Boundary Fragility}
\label{sec:boundary_fragility}

Uncertainty behavior is most fragile near decision boundaries. Let $s_0(x)\in\R^K$ and $s^Q_{\calD}(x)\in\R^K$ be the \fp and quantized option-score vectors, and define
$\Delta_{\calD}(x)=s^Q_{\calD}(x)-s_0(x)$. Let $i^*$ and $j^*$ be the \fp top and runner-up options, with margin
$\gamma(x)=s_0(x,i^*)-s_0(x,j^*)>0$.

\begin{proposition}[Top-two boundary fragility]
\label{prop:boundary}
If $\Delta_{\calD}(x,j^*)-\Delta_{\calD}(x,i^*)>\gamma(x)$, then the relative ordering of $i^*$ and $j^*$ is reversed by quantization. Conversely, if $\|\Delta_{\calD}(x)\|_\infty<\gamma(x)/2$, the \fp top-1 decision is preserved.
\end{proposition}

\noindent Thus, small-margin examples can flip under small perturbations, while large-margin examples are stable. This does not imply that \gptq directly optimizes option scores; it shows where reconstruction error becomes \emph{behavioral} error. Generic-text calibration may reduce average reconstruction error while missing activation directions that control answerability and confidence. Boundary-aware calibration exposes these fragile directions to the quantizer.

\subsection{Mixture Mismatch}
\label{sec:no_universal}

The two risks emphasize different parts of the input space. We make this explicit with a mixture view.

\paragraph{Mixture view.}
Let $q_{\mathrm{bdry}}$ denote the distribution over boundary or high-doubt examples, and let $q_{\mathrm{anchor}}$ denote the distribution over generic anchors, such as WikiText or random QA. A calibration recipe is $q_r = r\,q_{\mathrm{bdry}} + (1-r)\,q_{\mathrm{anchor}}$, where $r\in[0,1]$ is the boundary ratio used by \dpq; for example, \dpq-r75 corresponds to $r{=}0.75$. The target preservation distribution is
$T_\theta=\theta q_{\mathrm{bdry}}+(1-\theta)q_{\mathrm{anchor}}$. Answerability-boundary preservation has large $\theta$, while broad answerable \mcqa preservation has small $\theta$.

\paragraph{Fixed recipes.}
Generic-text calibration, such as WikiText or C4, corresponds roughly to $r\approx 0$. Boundary-only and answerability-only variants correspond to $r\approx 1$. Hard-example mining selects a different region: an example can be confidently wrong, with high NLL and large margin, without being uncertain; hence high-NLL $\neq$ high-doubt. Post-hoc calibration intervenes at a different stage: temperature maps preserve argmax and cannot repair flips, while flexible maps can change decisions but may move the model away from \fp behavior. These fixed choices cannot match every target; formal statements appear in Appendix~\ref{app:formal}.

\begin{proposition}[Mixture-ratio mismatch]
\label{prop:mismatch}
Let $\ell_{\calD}(x)\in[0,1]$ be a bounded preservation loss for calibration set $\calD$, and define $R_P(\calD)=\mathbb{E}_{x\sim P}[\ell_{\calD}(x)]$. For any two distributions $T$ and $q$,
\[
|R_T(\calD)-R_q(\calD)|\leq \mathrm{TV}(T,q).
\]
If $q_{\mathrm{bdry}}$ and $q_{\mathrm{anchor}}$ have disjoint support, then
\begin{equation}
\mathrm{TV}(T_\theta,q_r)=|\theta-r|,
\end{equation}
so the mismatch upper bound is minimized at $r^*=\theta$. Consequently, if two targets have different $\theta_1\neq\theta_2$, no single $r$ minimizes mismatch to both.
\end{proposition}

\noindent The proof and behavioral-risk surrogate appear in Appendix~\ref{app:finite_budget_tradeoff}. This result should be read as a design principle rather than an exact predictor of the best ratio for a given quantizer. It motivates evaluating target-dependent mixtures and interpreting the empirical winner as a target--recipe match, rather than expecting one fixed recipe to dominate all targets. \dpq provides this design space, and the experiments test the target--recipe match.

\section{Method}
\label{sec:method}

\begin{algorithm}[t]
\caption{\dpq Calibration Data Selection}
\label{alg:dpq}
\begin{algorithmic}[1]
\Require \fp model $M$; candidate pool $C$; calibration size $s$; mixture ratio $r$
\Ensure Calibration set $\calD$ with $|\calD|=s$
\State Score each $x\in C$ with $M$ to obtain $p_0(\cdot\mid x)$
\State Compute $b(x)=1-\big(p_{0,(1)}(x)-p_{0,(2)}(x)\big)$ for each $x\in C$
\State Select $C_{\mathrm{hi}}$ as the top $\lfloor sr\rfloor$ candidates by $b(x)$, with answerability balancing for \squad boundary data
\State Draw $C_{\mathrm{anc}}$ with $s-\lfloor sr\rfloor$ anchors from WikiText or RandomQA
\State Build $D_{\mathrm{hi}}$ from $C_{\mathrm{hi}}$ using the original prompt and \fp top option or options
\State Set $\calD \gets D_{\mathrm{hi}}\cup C_{\mathrm{anc}}$
\State Run unchanged \gptq with calibration set $\calD$
\end{algorithmic}
\end{algorithm}

\dpq is a pre-quantization calibration-data selection strategy for \gptq-style post-training quantization. Algorithm~\ref{alg:dpq} gives the procedure. \dpq does not modify the quantizer kernel, bit-width, reconstruction objective, or inference path; it only changes the calibration strings used to estimate activation statistics. We describe the method through the calibration-selection view in \S\ref{sec:method_problem} and the doubt-based criterion in \S\ref{sec:method_criterion}.

\subsection{Calibration Selection}
\label{sec:method_problem}

\gptq estimates layer-wise activation statistics from a calibration set $\calD$ to solve a local reconstruction problem, so $\calD$ determines which activation regions are represented accurately. We view calibration selection as choosing $\calD$ from a distribution $q$ that approximates the target preservation distribution $T_\theta$ from \S\ref{sec:no_universal}. Proposition~\ref{prop:mismatch} motivates allocating calibration mass according to the target mixture over boundary and anchor components.

Our candidate pool $C$ combines training-split ARC-Challenge examples, which provide answerable QA structure, with \squad answerability examples, which provide balanced answerable and unanswerable boundary cases. In our implementation, $C$ contains 512 candidates from each source before model-specific scoring. The \fp model scores $C$ once under the evaluation option-scoring protocol, producing $\{p_0(\cdot\mid x)\}_{x\in C}$; this is an offline selection cost and adds no inference-time overhead. Evaluation examples are never used for calibration, and the same pool is reused across all \dpq variants. This setup reflects the deployment view of calibration selection: the pool provides target-relevant candidates, while \dpq determines which examples within that pool best preserve the desired \fp behavior.

\subsection{Doubt-Based Selection}
\label{sec:method_criterion}

For a candidate $x$, let $p_{0,(1)}(x)$ and $p_{0,(2)}(x)$ denote the largest and second-largest \fp option probabilities. We define the doubt score
\begin{equation}
  b(x)=1-\big(p_{0,(1)}(x)-p_{0,(2)}(x)\big),
  \label{eq:boundary_score}
\end{equation}
which assigns high values to low-margin examples, the fragile region identified by Proposition~\ref{prop:boundary}. To isolate the effect of each signal, we also evaluate component variants that replace $b(x)$ with $1{-}p_{0,(1)}(x)$ for confidence-only, $H(p_0(\cdot\mid x))$ for entropy-only, answerability log-odds for answerability-only, or gold-label NLL for hard-example mining.

Given budget $s$ and mixture ratio $r\in[0,1]$, \dpq selects the top $\lfloor sr\rfloor$ candidates by $b(x)$ as boundary strings and draws the remaining $s-\lfloor sr\rfloor$ anchor strings from WikiText-style text and RandomQA-style calibration. Each boundary string concatenates the original prompt with the \fp top option or options, while anchors are kept unchanged. This construction is fixed across \dpq variants, so the ablations compare selection signal, ratio, and budget under the same calibration-string design. The ratio $r$ is the mixture weight from \S\ref{sec:no_universal}: higher $r$ targets boundary-heavy settings such as \squad answerability, and lower $r$ targets broader answerable behavior. We evaluate ratio, size, boundary, and component variants as controlled \dpq recipes in Section~\ref{sec:experiments}. Because \dpq only changes calibration strings, the resulting model has the same form as a standard \gptq baseline and incurs no additional inference cost. Thus, \dpq is best viewed as a target-aware recipe family rather than a single universal calibration set.

\section{Experiment Results}
\label{sec:experiments}
\subsection{Setup}

\paragraph{Evaluation scope.}
We evaluate 8 LMs across three model families and multiple scales: Qwen2.5-0.5B/1.5B/3B/7B \citep{yang2024qwen2}; Llama-3.2-1B, Llama-3.2-3B \citep{meta2024llama32}, and Llama-3.1-8B \citep{grattafiori2024llama3herd}; and Mistral-7B-v0.3 \citep{jiang2023mistral}. The benchmark suite contains 9 NLP datasets organized by preservation target. The old-core suite includes ARC-Challenge \citep{clark2018arc}, \squad answerability \citep{rajpurkar-etal-2018-know}, and TruthfulQA \citep{lin-etal-2022-truthfulqa}. Within this suite, \squad is the primary boundary-preservation benchmark because it directly tests answerable versus unanswerable decisions, while ARC-Challenge and TruthfulQA provide additional checks of score behavior. The extra-\mcqa suite includes ARC-Easy \citep{clark2018arc}, BoolQ \citep{clark-etal-2019-boolq}, PIQA \citep{bisk2020piqa}, HellaSwag \citep{zellers-etal-2019-hellaswag}, OpenBookQA \citep{mihaylov-etal-2018-suit}, and CommonsenseQA \citep{talmor-etal-2019-commonsenseqa}; since these examples are answerable, they primarily test broad distributional preservation.

\paragraph{Methods and ranking.}
For each model--dataset pair, the main comparison includes one full-precision reference and 22 comparison methods: \bnb{}; generic and task-formatted \gptq{} calibration such as WikiText, C4, RandomQA, and TaskRandom; \dpq{} ratio, size, boundary, and component variants; data-selection baselines; negative controls; and post-hoc controls. All main quantized baselines use 4-bit quantization; \bnb{} is included as an NF4 quantizer-family baseline, while the \gptq{} variants differ only in calibration data. Full method inventory, prompt formats, and quantization hyperparameters are provided in Appendix~\ref{app:repro}. We rank methods using target-aligned metrics: for \squad boundary preservation, boundary-accuracy deviation, answerability-rate deviation, \fp agreement, and JSD; for broad \mcqa preservation, \fp agreement, JSD, margin drift, confidence shift, accuracy deviation, and margin correlation. Within each target, metrics are equally weighted as a neutral summary because no deployment-specific utility function is assumed. Ranks are averaged over the 8 models and used only as compact summaries; our interpretations rely on component metrics and target-specific trends rather than on rank alone. Raw metrics and per-dataset breakdowns are in Appendix~\ref{app:full_results}.

\begin{table}[t]
\centering
\small
\begin{tabular}{@{}lcccc@{}}
\toprule
\multirow{2}{*}{Method} & \multicolumn{4}{c}{Worse than \fp{} (\%, $\downarrow$)} \\
\cmidrule(l){2-5}
& Acc. & ECE & NLL & Brier \\
\midrule
\bnb{}            & 84.7 & 66.7 & 84.7 & 81.9 \\
\gptq{}-RandomQA & 83.3 & 61.1 & 76.4 & 86.1 \\
\gptq{}-WikiText & 76.4 & 58.3 & 63.9 & 83.3 \\
\gptq{}-C4       & 83.3 & 65.3 & 68.1 & 77.8 \\
\midrule
\dpq{}-s128-r75  & 80.6 & 62.5 & 70.8 & 83.3 \\
\dpq{}-s128-r50  & 79.2 & 58.3 & 65.3 & 80.6 \\
\bottomrule
\end{tabular}
\caption{
Low-bit quantization often worsens accuracy and calibration metrics.
Entries show the percentage of model--dataset settings worse than \fp; lower is better.
Full results are in Table~\ref{tab:gap_full}.
}
\label{tab:gap}
\end{table}

\begin{figure}[t]
\centering
\includegraphics[width=\linewidth]{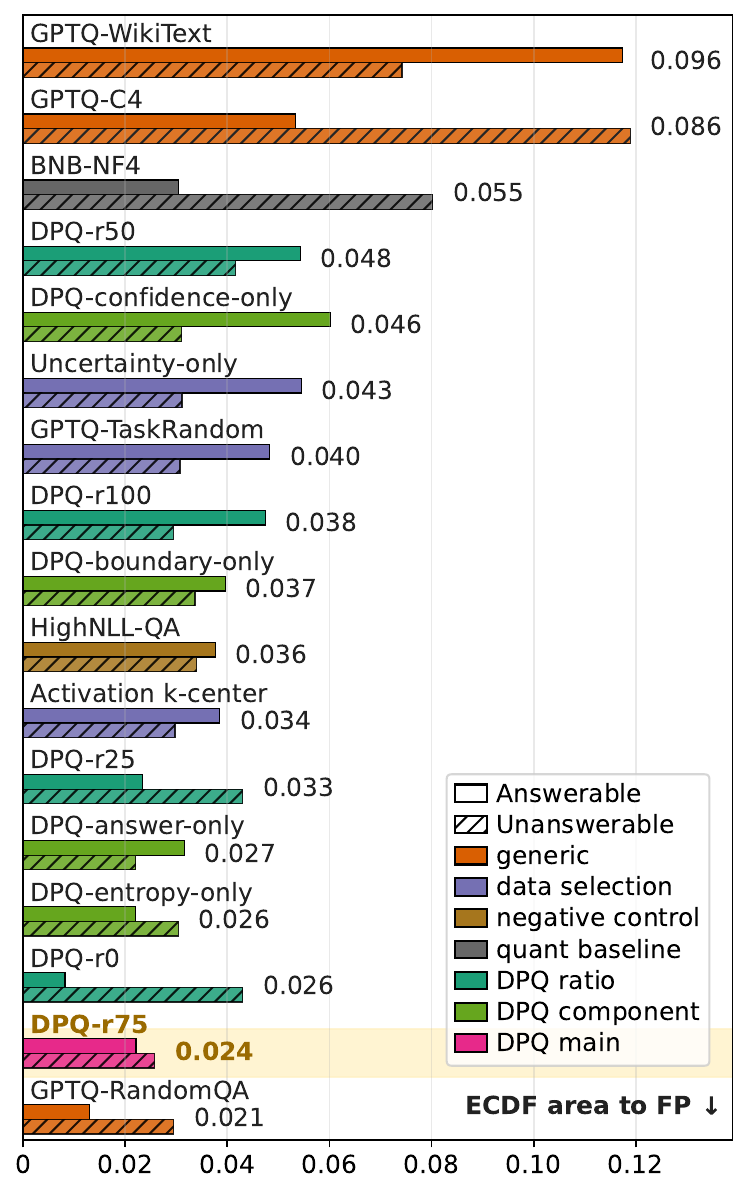}
\caption{
\squad answerability ECDF area to \fp per pre-quantization method, averaged over Qwen2.5-7B, Llama-3.1-8B, and Mistral-7B-v0.3. Numbers denote the mean of the answerable and unanswerable areas. Methods are sorted top-to-bottom by mean area (worst first); \dpq-r75, the best aggregate boundary recipe in Table~\ref{tab:squad}, is highlighted.
}
\label{fig:ecdf-area}
\end{figure}

\subsection{Quantization-Induced Drift}
\label{sec:drift}

Table~\ref{tab:gap} establishes the empirical problem. Across representative quantization and calibration choices, degradation is frequent not only in accuracy but also in calibration-sensitive metrics such as ECE, NLL, and Brier score. Thus, low-bit quantization often changes score behavior, not just final answers. The full degradation table, including confidence shift and margin drift, is provided in Table~\ref{tab:gap_full}. In particular, Table~\ref{tab:correct_wrong} in Appendix~\ref{app:additional_broad} shows that quantization compresses the confidence separation between correct and incorrect cases, although the direction of absolute confidence shifts is task-dependent.

This motivates the target-specific analysis that follows: \emph{given widespread quantization-induced drift, which calibration data choices best preserve the uncertainty behavior required by each deployment target?} We answer this separately for answerability-boundary preservation and broad \mcqa preservation.

\begin{figure*}[t]
\centering

\begin{minipage}[t]{0.62\textwidth}
\vspace{0pt}
\centering
\small
\resizebox{\linewidth}{!}{%
\begin{tabular}{@{}lccccc@{}}
\toprule
\multirow{2}{*}{Method} & \multirow{2}{*}{Rank $\downarrow$}
& \multicolumn{4}{c}{Boundary preservation metrics} \\
\cmidrule(l){3-6}
& & Agr. $\uparrow$ & Acc$\Delta$ $\downarrow$ & Rate$\Delta$ $\downarrow$ & JSD $\downarrow$ \\
\midrule
\textbf{\dpq{}-s128-r75} & \textbf{6.20} & \textbf{0.8495} & \textbf{0.2125} & \textbf{0.1000} & \textbf{0.0158} \\
\dpq{}-boundary-random    & 8.15  & 0.8413 & 0.2310 & 0.1135 & 0.0174 \\
\dpq{}-answerability-only & 8.18  & 0.8286 & 0.2592 & 0.1281 & 0.0205 \\
\dpq{}-entropy-only       & 8.68  & 0.8270 & 0.2940 & 0.1470 & 0.0227 \\
\dpq{}-boundary-only      & 9.06  & 0.8335 & 0.2520 & 0.1242 & 0.0197 \\
Uncertainty-only          & 9.36  & 0.8261 & 0.2937 & 0.1461 & 0.0181 \\
Activation k-center       & 9.39  & 0.8360 & 0.2530 & 0.1255 & 0.0187 \\
\dpq{}-s128-r0            & 10.31 & 0.8097 & 0.3205 & 0.1598 & 0.0220 \\
\gptq{}-RandomQA          & 11.09 & 0.8266 & 0.2733 & 0.1351 & 0.0231 \\
\dpq{}-s128-r100          & 12.44 & 0.8080 & 0.3340 & 0.1655 & 0.0215 \\
\dpq{}-s128-r50           & 13.12 & 0.8140 & 0.3210 & 0.1602 & 0.0258 \\
\gptq{}-WikiText          & 17.29 & 0.7446 & 0.4553 & 0.2271 & 0.0387 \\
\bottomrule
\end{tabular}}
\captionof{table}{
\squad answerability-boundary preservation.
Rank aggregates agreement with \fp (Agr.), boundary-accuracy deviation (Acc$\Delta$), answerability-rate deviation (Rate$\Delta$), and JSD over 8 models; lower is better except Agr.
Full ablations are in Table~\ref{tab:bdry_ablations}.
}
\label{tab:squad}
\end{minipage}
\hfill
\begin{minipage}[t]{0.36\textwidth}
\vspace{0pt}
\centering
\includegraphics[width=\linewidth]{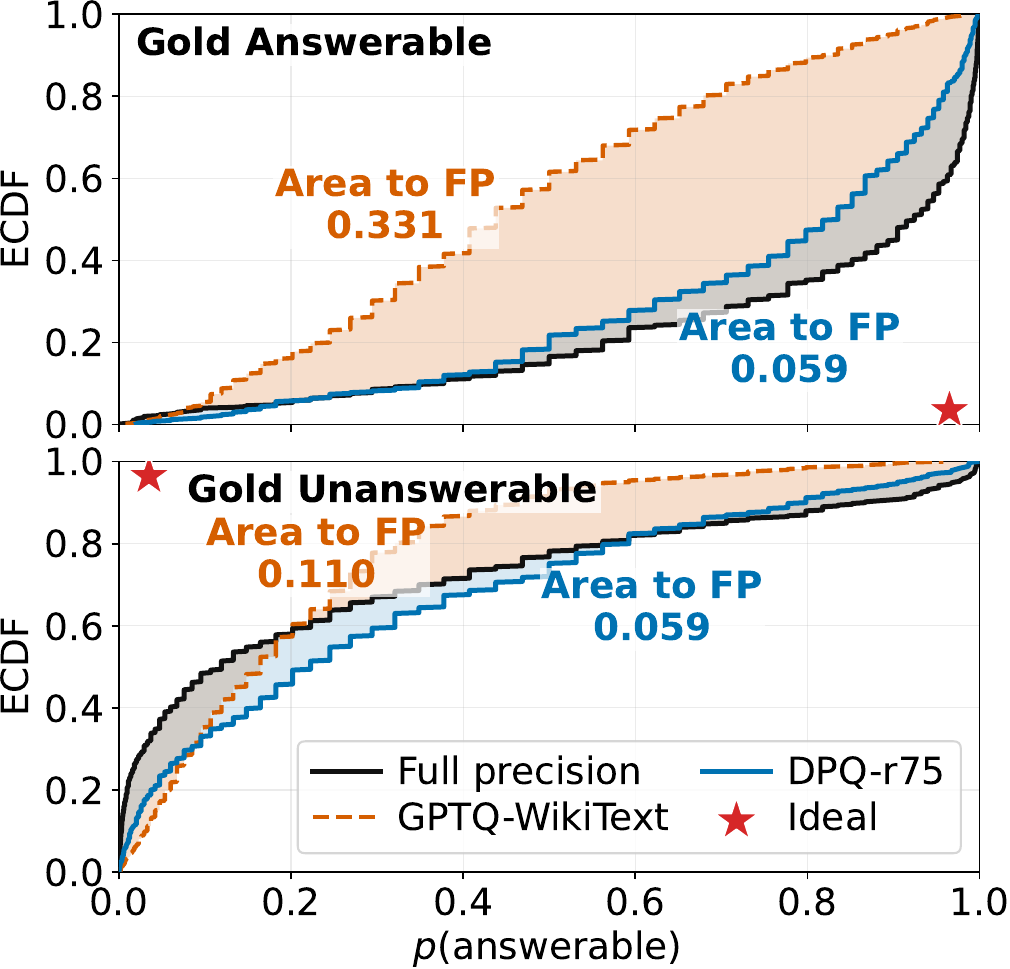}
\captionof{figure}{
\squad answerability ECDFs for Llama-3.1-8B.
Shaded regions show area to \fp; smaller is better.
Red stars mark ideal behavior.
}
\label{fig:llama-answerability-dist}
\end{minipage}

\end{figure*}

\subsection{Answerability Boundary}
\label{sec:squad_result}

Table~\ref{tab:squad} reports the core \squad answerability-boundary result. \dpq-s128-r75 is the strongest listed pre-quantization recipe, with the best aggregate rank, highest agreement with \fp, and smallest boundary-accuracy deviation, answerability-rate deviation, and JSD. Compared with \gptq-WikiText, it reduces boundary-accuracy deviation from 0.4553 to 0.2125, answerability-rate deviation from 0.2271 to 0.1000, and JSD from 0.0387 to 0.0158. The conclusion does not rely on rank alone: the same recipe also improves the absolute boundary and distributional metrics. Figure~\ref{fig:ecdf-area} gives a complementary deployment-scale view, showing that \dpq-r75 is among the recipes closest to the \fp answerability distribution. Together, these results show that boundary-heavy targets benefit from high-doubt calibration examples, while the mixed r75 recipe also indicates that generic anchors are needed. The ordering remains stable after removing the Llama-3.2-1B stress case; see Table~\ref{tab:squad_nollama}.

\paragraph{Distributional mechanism.}
Figure~\ref{fig:llama-answerability-dist} gives a distributional view of the \squad result on Llama-3.1-8B. For gold-answerable examples, $p(\mathrm{answerable})$ should remain near one; for gold-unanswerable examples, it should remain near zero. Generic \gptq-WikiText calibration shifts the answerability distribution away from \fp, especially on answerable examples, while \dpq-r75 moves it closer to the \fp reference. This illustrates why \squad answerability is not simply a harder \mcqa setting: an answerability flip changes whether the system should answer at all. The pattern is consistent with using a boundary-heavy mixture for this target, while the ablations below show that anchors are still needed for stability. Additional ECDF diagnostics for Qwen2.5-7B and Mistral-7B-v0.3 show the same qualitative pattern in Appendix~\ref{app:bdry_ablations}.

\begin{table}[t]
\centering
\small
\begingroup
\renewcommand{\tabularxcolumn}[1]{m{#1}}
\begin{tabularx}{\columnwidth}{
>{\centering\arraybackslash}m{0.15\columnwidth}
>{\raggedright\arraybackslash}X
}
\toprule
\textbf{Check} & \textbf{Main finding} \\
\midrule
\rowcolor{rowgray}
Boundary ratio $r$
& \dpq-r75 performs best. A high boundary ratio helps, but using only boundary examples over-concentrates the calibration set. \\
Calibration size $s$
& The composition of the calibration set matters more than simply increasing the calibration budget. \\
\rowcolor{rowgray}
Boundary mixing
& Mixed calibration outperforms boundary-only variants, suggesting that generic anchors stabilize boundary-focused selection. \\
Single-signal variants
& Confidence-only, entropy-only, and answerability-only variants are useful, but none matches the mixed \dpq-r75 recipe. \\
\rowcolor{rowgray}
Negative controls
& HighNLL-QA and LowDoubt-QA underperform \dpq-r75, showing that high doubt is different from ordinary hard- or easy-example selection. \\
\bottomrule
\end{tabularx}
\endgroup
\caption{
Ablation summary for the \squad boundary target.
Full results are in Appendix~\ref{app:bdry_ablations}.
}
\label{tab:ablation_summary}
\end{table}

\paragraph{Ablation summary.}
The ablations in Appendix~\ref{app:bdry_ablations} test ratio, size, boundary mixing, single-signal components, and negative controls. As summarized in Table~\ref{tab:ablation_summary}, the gains mainly come from calibration-set composition rather than budget size: r75 performs best, r100 over-concentrates on boundary cases, and boundary-only or single-signal variants do not match the mixed recipe. Strong data-selection baselines improve over generic text calibration but still fall below \dpq-r75 on the boundary-specific aggregate. These ablations test the main selection axes of the calibration set; the boundary-string construction is held fixed as part of the \dpq recipe.

\begin{figure}[t]
\centering
\includegraphics[width=\columnwidth]{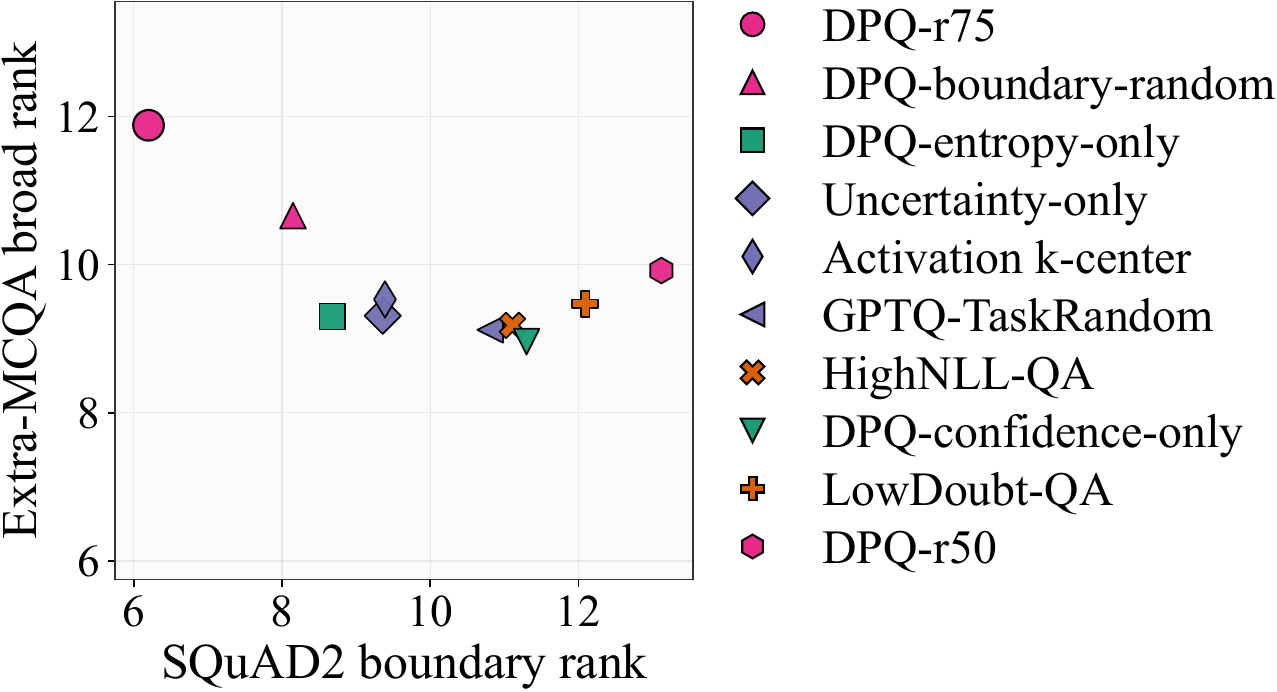}
\caption{
Target-dependent preservation trade-off: \squad boundary rank versus extra-\mcqa broad-preservation rank. Lower is better on both axes; no calibration recipe dominates both targets.
}
\label{fig:tradeoff}
\end{figure}

\subsection{Broad MCQA Trade-off}
\label{sec:mcqa_result}

The six extra-\mcqa benchmarks contain only answerable examples, so the relevant target is broad \fp-behavior preservation ($\mathcal{R}_{\mathrm{dist}}$), not answerability-boundary preservation. Table~\ref{tab:extra} shows that the leading recipes therefore shift away from the high-boundary setting: confidence-only, task-random, HighNLL-QA, entropy-only, uncertainty-only, and \dpq-r50 are more competitive than \dpq-r75. This pattern matches the mixture view: when the target places less mass on answerability boundaries, a smaller boundary ratio or a single uncertainty signal can better preserve broad option-distribution behavior. Figure~\ref{fig:tradeoff} visualizes the same target dependence: \dpq-r75 is stronger on the \squad boundary axis, whereas milder or single-signal recipes are stronger on broad \mcqa. The metric columns also show why Table~\ref{tab:extra} should be read as a trade-off rather than a single scalar leaderboard: \bnb has the strongest agreement and JSD, whereas other recipes rank higher under the full multi-metric objective. Expanded extra-\mcqa results in Appendix~\ref{app:broad_pres} support the target-dependent recipe shift, while the calibration-gap and correct-vs-wrong confidence diagnostics in Appendix~\ref{app:additional_broad} further show that quantization can alter score behavior in ways not captured by top-1 accuracy alone.

\begin{table}[t]
\centering
\small
\resizebox{\columnwidth}{!}{%
\begin{tabular}{@{}lcccc@{}}
\toprule
Method & Rank $\downarrow$ & Top-5 $\uparrow$ & Agr. $\uparrow$ & JSD $\downarrow$ \\
\midrule
\dpq{}-confidence-only & 8.96 & 7 & 0.8357 & 0.0284 \\
\gptq{}-TaskRandom & 9.12 & 11 & 0.8361 & 0.0276 \\
HighNLL-QA & 9.18 & 6 & 0.8327 & 0.0294 \\
\dpq{}-entropy-only & 9.30 & 10 & 0.8324 & 0.0296 \\
Uncertainty-only & 9.31 & 4 & 0.8318 & 0.0297 \\
LowDoubt-QA & 9.47 & 9 & 0.8328 & 0.0286 \\
Activation k-center & 9.53 & 12 & 0.8349 & 0.0299 \\
\dpq{}-s128-r50 & 9.92 & 4 & 0.8322 & 0.0293 \\
\bnb{} & 10.53 & 12 & 0.8637 & 0.0241 \\
\dpq{}-s128-r75 & 11.88 & 3 & 0.8180 & 0.0324 \\
\bottomrule
\end{tabular}}
\caption{
Broad \mcqa preservation on six answerable datasets.
Rank aggregates agreement with \fp (Agr.), JSD, margin drift, confidence shift, accuracy deviation, and margin correlation; Top-5 counts model--dataset pairs where the method ranks in the top five.
Milder, single-signal, or broad-selection recipes lead, unlike on the \squad boundary target.
}
\label{tab:extra}
\end{table}

\paragraph{Practical takeaway.}
In deployment, the recipe should be chosen by the preservation metric that matches the downstream use. If validation examples reflect answerability, abstention, or other low-margin decisions, a boundary-heavy mixed recipe such as \dpq-r75 is a strong choice. If the target mainly uses broad option scores on answerable \mcqa tasks, milder or single-signal recipes are more appropriate. Thus, \dpq is best used as a target-aware recipe family, not as a universal calibration set selected by task accuracy alone.

\paragraph{Negative controls.}
The negative controls further clarify the target shift. HighNLL-QA and LowDoubt-QA are competitive on broad \mcqa but not on \squad, matching the distinction between difficulty and doubt formalized in Appendix~\ref{app:hard_example_theory}. Thus, ordinary hard or easy examples can overlap with broad answerable uncertainty, but they are not answerability-boundary examples.

\paragraph{Post-hoc calibration targets a different objective.}
Post-hoc calibration targets label-oriented calibration or task accuracy, whereas our objective is preserving \fp score behavior; Table~\ref{tab:posthoc} illustrates this distinction. Adaptive temperature leaves accuracy and agreement with \fp unchanged, as expected from Proposition~\ref{prop:temp_argmax}, but increases JSD and margin drift. Flexible score-space calibrators improve \mcqa accuracy, yet reduce agreement with \fp and increase distributional drift. Thus, post-hoc calibration and pre-quantization selection are complementary intervention points: the former reshapes scores after quantization, while the latter controls which \fp score behavior the quantized model tends to preserve. The optimally fitted temperature check in Appendix~\ref{app:additional_posthoc} supports the same complementary pattern.

\begin{table}[t]
\centering
\small
\begin{tabular}{@{}lcccc@{}}
\toprule
Method & Acc. $\uparrow$ & Agr. $\uparrow$ & JSD $\downarrow$ & Margin$\Delta$ $\downarrow$ \\
\midrule
Base & 0.6833 & 0.8245 & 0.0320 & 0.1533 \\
\midrule
Adaptive temp. & 0.6833 & 0.8245 & 0.0445 & 0.2385 \\
Option bias & 0.6985 & 0.8147 & 0.0328 & 0.1591 \\
Vector & 0.7029 & 0.7773 & 0.0519 & 0.2344 \\
Matrix & 0.6992 & 0.7723 & 0.0550 & 0.2380 \\
Dirichlet & 0.6989 & 0.7740 & 0.0548 & 0.2400 \\
Isotonic & 0.6991 & 0.7710 & 0.0561 & 0.2329 \\
\bottomrule
\end{tabular}
\caption{
Post-hoc calibration on extra-\mcqa.
Flexible calibrators improve accuracy but reduce agreement with \fp and increase distributional drift.
}
\label{tab:posthoc}
\end{table}

\paragraph{AWQ transfer.}
The AWQ extension serves as a quantizer-family scope check. On 7B/8B-scale models, \dpq-r75 leads on AWQ mean rank, top-1 frequency, and agreement with \fp; its mean rank is 1.97, with 42.9\% top-1 and 73.0\% top-2 frequency. However, the per-dataset breakdown in Table~\ref{tab:awq-per-dataset} shows that score-based metrics can favor generic calibration on some datasets. Thus, the target-aware lens partially transfers, but the exact recipe should be tuned to the quantizer and metric; the per-dataset breakdown is in Appendix~\ref{app:awq-extension} and aggregate summaries are in Appendix~\ref{app:additional_awq}.

\section{Conclusion}

Quantization should not be evaluated only by whether it preserves a model's top-1 answer. In applications that rely on confidence, margins, selective prediction, abstention, or reranking, a quantized model may keep the same prediction while changing the uncertainty behavior used downstream. We study calibration-data selection as a target-dependent uncertainty-preservation problem and show that different targets favor different recipes: high-boundary mixtures such as \dpq-r75 better preserve \squad-style answerability boundaries, while milder mixtures or single-signal variants better preserve broad \mcqa behavior. The central insight is therefore not that one calibration set is universally best, but that calibration data should be selected for the specific \fp score behavior a deployment needs to preserve. More broadly, pre-quantization data selection, post-quantization calibration, and quantizer choice affect preservation differently, so recipes should be tuned to both the target behavior and the quantization method.

\section*{Limitations}
\label{sec:limitations}

\paragraph{Option-scoring uncertainty.}
We operationalize uncertainty through option scoring and answerability. This controlled setting allows confidence, margins, and full-precision agreement to be measured directly, and it covers common uses such as multiple-choice QA, selective prediction, abstention, and reranking. Extending the same preservation perspective to verbalized confidence and long-form generation is an important direction for future work, since uncertainty in those settings may appear in the generated text rather than in option probabilities.

\paragraph{Target-oriented candidate pool.}
Our candidate pool is intentionally target-oriented: it includes answerable QA examples and answerability-style training cases because the main boundary target is answerable/unanswerable preservation. This design matches the goal of studying deployment-aware calibration selection, where the calibration pool should reflect the behavior one aims to preserve. Future work can further test less aligned or domain-shifted candidate pools to study how broadly the same selection principles transfer; in deployment, target-specific gains should also be checked against broad-preservation metrics to ensure that the selected recipe remains appropriate beyond the primary target.

\paragraph{Quantizer and bit-width scope.}
Our main experiments focus on \gptq-style calibration-data selection because \gptq directly uses calibration strings to estimate activation statistics. We include \textsc{AWQ} and \bnb as quantizer-family checks, and the results suggest that some preservation signals transfer while score-based metrics remain quantizer-specific. A broader sweep over quantizers and bit-widths would further refine deployment-specific calibration recipes.

\paragraph{Interaction with post-hoc calibration.}
Our post-hoc analysis covers representative score-space methods to distinguish pre-quantization data selection from post-quantization score repair. These experiments show that the two stages can optimize different objectives and may be complementary. A fuller study of how to combine target-aware calibration selection with post-hoc calibration is a useful direction for future work.

\paragraph{Deployment risk.}
Uncertainty drift after quantization can affect downstream decisions in ways that are not visible from top-1 accuracy alone. In abstention, triage, educational QA, medical or legal assistance, and safety-filtering settings, shifted confidence or answerability boundaries may cause a system to answer when it should defer, or to suppress useful answers. This reinforces the need to audit quantized models not only for accuracy but also for the score behavior consumed by downstream decision modules.

\section*{Acknowledgments}
This work was supported by the Start-up Grant of City University of Hong Kong (Dongguan). Generative AI tools were used only for language polishing and grammatical correction.

\bibliography{ref}

\begin{thebibliography}{55}
\providecommand{\natexlab}[1]{#1}

\bibitem[{Ash et~al.(2020)Ash, Zhang, Krishnamurthy, Langford, and Agarwal}]{ash2020badge}
Jordan~T. Ash, Chicheng Zhang, Akshay Krishnamurthy, John Langford, and Alekh Agarwal. 2020.
\newblock \href {https://openreview.net/forum?id=ryghZJBKPS} {Deep batch active learning by diverse, uncertain gradient lower bounds}.
\newblock In \emph{International Conference on Learning Representations}.

\bibitem[{Ashkboos et~al.(2024)Ashkboos, Mohtashami, Croci, Li, Cameron, Jaggi, Alistarh, Hoefler, and Hensman}]{ashkboos2024quarot}
Saleh Ashkboos, Amirkeivan Mohtashami, Maximilian~L. Croci, Bo~Li, Pashmina Cameron, Martin Jaggi, Dan Alistarh, Torsten Hoefler, and James Hensman. 2024.
\newblock \href {https://proceedings.neurips.cc/paper_files/paper/2024/hash/b5b939436789f76f08b9d0da5e81af7c-Abstract-Conference.html} {{QuaRot}: Outlier-free 4-bit inference in rotated {LLM}s}.
\newblock In \emph{Advances in Neural Information Processing Systems}.

\bibitem[{Bisk et~al.(2020)Bisk, Zellers, Le~Bras, Gao, and Choi}]{bisk2020piqa}
Yonatan Bisk, Rowan Zellers, Ronan Le~Bras, Jianfeng Gao, and Yejin Choi. 2020.
\newblock \href {https://ojs.aaai.org/index.php/AAAI/article/view/6239} {{PIQA}: Reasoning about physical commonsense in natural language}.
\newblock In \emph{Proceedings of the Thirty-Fourth AAAI Conference on Artificial Intelligence}, pages 7432--7439.

\bibitem[{Brier(1950)}]{brier1950verification}
Glenn~W. Brier. 1950.
\newblock \href {https://doi.org/10.1175/1520-0493(1950)078<0001:VOFEIT>2.0.CO;2} {Verification of forecasts expressed in terms of probability}.
\newblock \emph{Monthly Weather Review}, 78(1):1--3.

\bibitem[{Chow(1970)}]{chow1970optimum}
C.~K. Chow. 1970.
\newblock \href {https://doi.org/10.1109/TIT.1970.1054406} {On optimum recognition error and reject tradeoff}.
\newblock \emph{IEEE Transactions on Information Theory}, 16(1):41--46.

\bibitem[{Clark et~al.(2019)Clark, Lee, Chang, Kwiatkowski, Collins, and Toutanova}]{clark-etal-2019-boolq}
Christopher Clark, Kenton Lee, Ming-Wei Chang, Tom Kwiatkowski, Michael Collins, and Kristina Toutanova. 2019.
\newblock \href {https://doi.org/10.18653/v1/N19-1300} {{B}ool{Q}: Exploring the surprising difficulty of natural yes/no questions}.
\newblock In \emph{Proceedings of the 2019 Conference of the North {A}merican Chapter of the Association for Computational Linguistics: Human Language Technologies, Volume 1 (Long and Short Papers)}, pages 2924--2936, Minneapolis, Minnesota. Association for Computational Linguistics.

\bibitem[{Clark et~al.(2018)Clark, Cowhey, Etzioni, Khot, Sabharwal, Schoenick, and Tafjord}]{clark2018arc}
Peter Clark, Isaac Cowhey, Oren Etzioni, Tushar Khot, Ashish Sabharwal, Carissa Schoenick, and Oyvind Tafjord. 2018.
\newblock \href {https://arxiv.org/abs/1803.05457} {Think you have solved question answering? try {ARC}, the {AI2} reasoning challenge}.
\newblock \emph{arXiv preprint arXiv:1803.05457}.

\bibitem[{Desai and Durrett(2020)}]{desai-durrett-2020-calibration}
Shrey Desai and Greg Durrett. 2020.
\newblock \href {https://doi.org/10.18653/v1/2020.emnlp-main.21} {Calibration of pre-trained transformers}.
\newblock In \emph{Proceedings of the 2020 Conference on Empirical Methods in Natural Language Processing (EMNLP)}, pages 295--302, Online. Association for Computational Linguistics.

\bibitem[{Dettmers et~al.(2022)Dettmers, Lewis, Belkada, and Zettlemoyer}]{dettmers2022llmint8}
Tim Dettmers, Mike Lewis, Younes Belkada, and Luke Zettlemoyer. 2022.
\newblock \href {https://arxiv.org/abs/2208.07339} {{LLM.int8()}: 8-bit matrix multiplication for transformers at scale}.
\newblock In \emph{Advances in Neural Information Processing Systems}.

\bibitem[{Dettmers et~al.(2023)Dettmers, Pagnoni, Holtzman, and Zettlemoyer}]{dettmers2023qlora}
Tim Dettmers, Artidoro Pagnoni, Ari Holtzman, and Luke Zettlemoyer. 2023.
\newblock \href {https://arxiv.org/abs/2305.14314} {{QLoRA}: Efficient finetuning of quantized {LLM}s}.
\newblock In \emph{Advances in Neural Information Processing Systems}.

\bibitem[{Dettmers et~al.(2024)Dettmers, Svirschevski, Egiazarian, Kuznedelev, Frantar, Ashkboos, Borzunov, Hoefler, and Alistarh}]{dettmers2024spqr}
Tim Dettmers, Ruslan Svirschevski, Vage Egiazarian, Denis Kuznedelev, Elias Frantar, Saleh Ashkboos, Alexander Borzunov, Torsten Hoefler, and Dan Alistarh. 2024.
\newblock \href {https://openreview.net/forum?id=Q1u25ahSuy} {{SpQR}: A sparse-quantized representation for near-lossless {LLM} weight compression}.
\newblock In \emph{International Conference on Learning Representations}.

\bibitem[{Egiazarian et~al.(2024)Egiazarian, Panferov, Kuznedelev, Frantar, Babenko, and Alistarh}]{egiazarian2024aqlm}
Vage Egiazarian, Andrei Panferov, Denis Kuznedelev, Elias Frantar, Artem Babenko, and Dan Alistarh. 2024.
\newblock \href {https://proceedings.mlr.press/v235/egiazarian24a.html} {Extreme compression of large language models via additive quantization}.
\newblock In \emph{Proceedings of the 41st International Conference on Machine Learning}, volume 235 of \emph{Proceedings of Machine Learning Research}, pages 12284--12303. PMLR.

\bibitem[{El-Yaniv and Wiener(2010)}]{elyaniv2010foundations}
Ran El-Yaniv and Yair Wiener. 2010.
\newblock \href {https://jmlr.org/papers/v11/el-yaniv10a.html} {On the foundations of noise-free selective classification}.
\newblock \emph{Journal of Machine Learning Research}, 11:1605--1641.

\bibitem[{Frantar et~al.(2023)Frantar, Ashkboos, Hoefler, and Alistarh}]{frantar2023gptq}
Elias Frantar, Saleh Ashkboos, Torsten Hoefler, and Dan Alistarh. 2023.
\newblock \href {https://arxiv.org/abs/2210.17323} {{GPTQ}: Accurate post-training quantization for generative pre-trained transformers}.
\newblock In \emph{International Conference on Learning Representations}.

\bibitem[{Geifman and El-Yaniv(2017)}]{geifman2017selective}
Yonatan Geifman and Ran El-Yaniv. 2017.
\newblock \href {https://arxiv.org/abs/1705.08500} {Selective classification for deep neural networks}.
\newblock In \emph{Advances in Neural Information Processing Systems}.

\bibitem[{Grattafiori et~al.(2024)Grattafiori, Dubey, Jauhri, Pandey, Kadian, Al-Dahle, Letman, Mathur, Schelten, Vaughan, Yang, Fan et~al.}]{grattafiori2024llama3herd}
Aaron Grattafiori, Abhimanyu Dubey, Abhinav Jauhri, Abhinav Pandey, Abhishek Kadian, Ahmad Al-Dahle, Aiesha Letman, Akhil Mathur, Alan Schelten, Alex Vaughan, Amy Yang, Angela Fan, and 1 others. 2024.
\newblock \href {https://doi.org/10.48550/arXiv.2407.21783} {The {Llama 3} herd of models}.
\newblock \emph{arXiv preprint arXiv:2407.21783}.

\bibitem[{Guo et~al.(2017)Guo, Pleiss, Sun, and Weinberger}]{guo2017calibration}
Chuan Guo, Geoff Pleiss, Yu~Sun, and Kilian~Q. Weinberger. 2017.
\newblock \href {https://proceedings.mlr.press/v70/guo17a.html} {On calibration of modern neural networks}.
\newblock In \emph{Proceedings of the 34th International Conference on Machine Learning}, pages 1321--1330.

\bibitem[{He et~al.(2026)He, Wen, Zhan, Chen, Cui, Lan, and Wang}]{he2026budgetdraft}
Liang He, Jingbo Wen, Qishi Zhan, Yixiong Chen, Kangning Cui, Qizhen Lan, and Xilu Wang. 2026.
\newblock Budgetdraft: Acceptance-aware multi-view training for sparse-kv speculative decoding.
\newblock \emph{arXiv preprint arXiv:2606.00144}.

\bibitem[{Jiang et~al.(2023)Jiang, Sablayrolles, Mensch, Bamford, Chaplot, Casas, Bressand, Lengyel, Lample, Saulnier, Lavaud, Lachaux, Stock, Le~Scao, Lavril, Wang, Lacroix, and El~Sayed}]{jiang2023mistral}
Albert~Q. Jiang, Alexandre Sablayrolles, Arthur Mensch, Chris Bamford, Devendra~Singh Chaplot, Diego de~las Casas, Florian Bressand, Gianna Lengyel, Guillaume Lample, Lucile Saulnier, L{\'e}lio~Renard Lavaud, Marie-Anne Lachaux, Pierre Stock, Teven Le~Scao, Thibaut Lavril, Thomas Wang, Timoth{\'e}e Lacroix, and William El~Sayed. 2023.
\newblock \href {https://arxiv.org/abs/2310.06825} {Mistral 7b}.
\newblock \emph{arXiv preprint arXiv:2310.06825}.

\bibitem[{Jiang et~al.(2020)Jiang, Xu, Araki, and Neubig}]{jiang-etal-2020-know}
Zhengbao Jiang, Frank~F. Xu, Jun Araki, and Graham Neubig. 2020.
\newblock \href {https://doi.org/10.1162/tacl_a_00324} {How can we know what language models know?}
\newblock \emph{Transactions of the Association for Computational Linguistics}, 8:423--438.

\bibitem[{Kadavath et~al.(2022)Kadavath, Conerly, Askell, Henighan, Drain, Perez, Schiefer, Hatfield-Dodds, DasSarma, Tran-Johnson, Johnston, El-Showk, Jones, Elhage, Hume, Chen, Bai, Bowman, Fort, Ganguli, Hernandez, Jacobson, Kernion, Kravec, Lovitt, Ndousse, Olsson, Ringer, Amodei, Brown, Clark, Joseph, Mann, McCandlish, Olah, and Kaplan}]{kadavath2022language}
Saurav Kadavath, Tom Conerly, Amanda Askell, Tom Henighan, Dawn Drain, Ethan Perez, Nicholas Schiefer, Zac Hatfield-Dodds, Nova DasSarma, Eli Tran-Johnson, Scott Johnston, Sheer El-Showk, Andy Jones, Nelson Elhage, Tristan Hume, Anna Chen, Yuntao Bai, Sam Bowman, Stanislav Fort, and 17 others. 2022.
\newblock \href {https://arxiv.org/abs/2207.05221} {Language models (mostly) know what they know}.
\newblock \emph{arXiv preprint arXiv:2207.05221}.

\bibitem[{Kapoor et~al.(2024{\natexlab{a}})Kapoor, Gruver, Roberts, Collins, Pal, Bhatt, Weller, Dooley, Goldblum, and Wilson}]{kapoor2024large}
Sanyam Kapoor, Nate Gruver, Manley Roberts, Katherine Collins, Arka Pal, Umang Bhatt, Adrian Weller, Samuel Dooley, Micah Goldblum, and Andrew~G Wilson. 2024{\natexlab{a}}.
\newblock Large language models must be taught to know what they don’t know.
\newblock \emph{Advances in Neural Information Processing Systems}, 37:85932--85972.

\bibitem[{Kapoor et~al.(2024{\natexlab{b}})Kapoor, Gruver, Roberts, Pal, Dooley, Goldblum, and Wilson}]{kapoor2024calibration}
Sanyam Kapoor, Nate Gruver, Manley Roberts, Arka Pal, Samuel Dooley, Micah Goldblum, and Andrew Wilson. 2024{\natexlab{b}}.
\newblock Calibration-tuning: Teaching large language models to know what they don’t know.
\newblock In \emph{Proceedings of the 1st Workshop on Uncertainty-Aware NLP (UncertaiNLP 2024)}, pages 1--14.

\bibitem[{Kim et~al.(2024)Kim, Hooper, Gholami, Dong, Li, Shen, Mahoney, and Keutzer}]{kim2024squeezellm}
Sehoon Kim, Coleman Richard~Charles Hooper, Amir Gholami, Zhen Dong, Xiuyu Li, Sheng Shen, Michael~W. Mahoney, and Kurt Keutzer. 2024.
\newblock \href {https://proceedings.mlr.press/v235/kim24f.html} {{SqueezeLLM}: Dense-and-sparse quantization}.
\newblock In \emph{Proceedings of the 41st International Conference on Machine Learning}, volume 235 of \emph{Proceedings of Machine Learning Research}, pages 23901--23923. PMLR.

\bibitem[{Kull et~al.(2019)Kull, Perello~Nieto, K{\"a}ngsepp, Silva~Filho, Song, and Flach}]{kull2019dirichlet}
Meelis Kull, Miquel Perello~Nieto, Markus K{\"a}ngsepp, Telmo Silva~Filho, Hao Song, and Peter Flach. 2019.
\newblock \href {https://proceedings.neurips.cc/paper/2019/hash/8ca01ea920679a0fe3728441494041b9-Abstract.html} {Beyond temperature scaling: Obtaining well-calibrated multiclass probabilities with dirichlet calibration}.
\newblock In \emph{Advances in Neural Information Processing Systems}.

\bibitem[{Kumar et~al.(2019)Kumar, Liang, and Ma}]{kumar2019verified}
Ananya Kumar, Percy Liang, and Tengyu Ma. 2019.
\newblock \href {https://proceedings.neurips.cc/paper_files/paper/2019/hash/f8c0c968632845cd133308b1a494967f-Abstract.html} {Verified uncertainty calibration}.
\newblock In \emph{Advances in Neural Information Processing Systems}.

\bibitem[{Lin et~al.(2024)Lin, Tang, Tang, Yang, Chen, Wang, Xiao, Dang, Gan, and Han}]{lin2023awq}
Ji~Lin, Jiaming Tang, Haotian Tang, Shang Yang, Wei-Ming Chen, Wei-Chen Wang, Guangxuan Xiao, Xingyu Dang, Chuang Gan, and Song Han. 2024.
\newblock \href {https://proceedings.mlsys.org/paper_files/paper/2024/hash/42a452cbafa9dd64e9ba4aa95cc1ef21-Abstract-Conference.html} {{AWQ}: Activation-aware weight quantization for on-device {LLM} compression and acceleration}.
\newblock In \emph{Proceedings of Machine Learning and Systems}, volume~6.

\bibitem[{Lin et~al.(2022)Lin, Hilton, and Evans}]{lin-etal-2022-truthfulqa}
Stephanie Lin, Jacob Hilton, and Owain Evans. 2022.
\newblock \href {https://doi.org/10.18653/v1/2022.acl-long.229} {{T}ruthful{QA}: Measuring how models mimic human falsehoods}.
\newblock In \emph{Proceedings of the 60th Annual Meeting of the Association for Computational Linguistics (Volume 1: Long Papers)}, pages 3214--3252, Dublin, Ireland. Association for Computational Linguistics.

\bibitem[{{Meta AI}(2024)}]{meta2024llama32}
{Meta AI}. 2024.
\newblock \href {https://github.com/meta-llama/llama-models/blob/main/models/llama3_2/MODEL_CARD.md} {{Llama 3.2} model card}.
\newblock Model card for the Llama 3.2 text-only model collection.

\bibitem[{Mihaylov et~al.(2018)Mihaylov, Clark, Khot, and Sabharwal}]{mihaylov-etal-2018-suit}
Todor Mihaylov, Peter Clark, Tushar Khot, and Ashish Sabharwal. 2018.
\newblock \href {https://doi.org/10.18653/v1/D18-1260} {Can a suit of armor conduct electricity? a new dataset for open book question answering}.
\newblock In \emph{Proceedings of the 2018 Conference on Empirical Methods in Natural Language Processing}, pages 2381--2391, Brussels, Belgium. Association for Computational Linguistics.

\bibitem[{Minderer et~al.(2021)Minderer, Djolonga, Romijnders, Hubis, Zhai, Houlsby, Tran, and Lucic}]{minderer2021revisiting}
Matthias Minderer, Josip Djolonga, Rob Romijnders, Frances~Ann Hubis, Xiaohua Zhai, Neil Houlsby, Dustin Tran, and Mario Lucic. 2021.
\newblock \href {https://proceedings.neurips.cc/paper/2021/hash/8420d359404024567b5aefda1231af24-Abstract.html} {Revisiting the calibration of modern neural networks}.
\newblock In \emph{Advances in Neural Information Processing Systems}.

\bibitem[{Naeini et~al.(2015)Naeini, Cooper, and Hauskrecht}]{naeini2015obtaining}
Mahdi~Pakdaman Naeini, Gregory~F. Cooper, and Milos Hauskrecht. 2015.
\newblock \href {https://ojs.aaai.org/index.php/AAAI/article/view/9602} {Obtaining well calibrated probabilities using bayesian binning}.
\newblock In \emph{Proceedings of the Twenty-Ninth AAAI Conference on Artificial Intelligence}, pages 2901--2907.

\bibitem[{Niculescu-Mizil and Caruana(2005)}]{niculescu2005predicting}
Alexandru Niculescu-Mizil and Rich Caruana. 2005.
\newblock \href {https://doi.org/10.1145/1102351.1102430} {Predicting good probabilities with supervised learning}.
\newblock In \emph{Proceedings of the 22nd International Conference on Machine Learning}, pages 625--632.

\bibitem[{Nixon et~al.(2019)Nixon, Dusenberry, Jerfel, Nguyen, Liu, Zhang, and Tran}]{nixon2019measuring}
Jeremy Nixon, Michael~W. Dusenberry, Ghassen Jerfel, Timothy Nguyen, Jeremiah Liu, Linchuan Zhang, and Dustin Tran. 2019.
\newblock \href {https://doi.org/10.48550/arXiv.1904.01685} {Measuring calibration in deep learning}.
\newblock In \emph{CVPR Workshops}.

\bibitem[{Ovadia et~al.(2019)Ovadia, Fertig, Ren, Nado, Sculley, Nowozin, Dillon, Lakshminarayanan, and Snoek}]{ovadia2019trust}
Yaniv Ovadia, Emily Fertig, Jie Ren, Zachary Nado, D.~Sculley, Sebastian Nowozin, Joshua~V. Dillon, Balaji Lakshminarayanan, and Jasper Snoek. 2019.
\newblock \href {https://proceedings.neurips.cc/paper/2019/hash/8558cb408c1d76621371888657d2eb1d-Abstract.html} {Can you trust your model's uncertainty? evaluating predictive uncertainty under dataset shift}.
\newblock In \emph{Advances in Neural Information Processing Systems}.

\bibitem[{Platt(1999)}]{platt1999probabilistic}
John~C. Platt. 1999.
\newblock Probabilistic outputs for support vector machines and comparisons to regularized likelihood methods.
\newblock \emph{Advances in Large Margin Classifiers}, pages 61--74.

\bibitem[{Proskurina et~al.(2024)Proskurina, Brun, Metzler, and Velcin}]{proskurina-etal-2024-quantization}
Irina Proskurina, Luc Brun, Guillaume Metzler, and Julien Velcin. 2024.
\newblock \href {https://doi.org/10.18653/v1/2024.findings-naacl.124} {When quantization affects confidence of large language models?}
\newblock In \emph{Findings of the Association for Computational Linguistics: NAACL 2024}, pages 1918--1928, Mexico City, Mexico. Association for Computational Linguistics.

\bibitem[{{Qwen Team}(2024)}]{yang2024qwen2}
{Qwen Team}. 2024.
\newblock \href {https://doi.org/10.48550/arXiv.2412.15115} {{Qwen2.5} technical report}.
\newblock \emph{arXiv preprint arXiv:2412.15115}.

\bibitem[{Rajpurkar et~al.(2018)Rajpurkar, Jia, and Liang}]{rajpurkar-etal-2018-know}
Pranav Rajpurkar, Robin Jia, and Percy Liang. 2018.
\newblock \href {https://doi.org/10.18653/v1/P18-2124} {Know what you don{'}t know: Unanswerable questions for {SQ}u{AD}}.
\newblock In \emph{Proceedings of the 56th Annual Meeting of the Association for Computational Linguistics (Volume 2: Short Papers)}, pages 784--789, Melbourne, Australia. Association for Computational Linguistics.

\bibitem[{Sener and Savarese(2018)}]{sener2018active}
Ozan Sener and Silvio Savarese. 2018.
\newblock \href {https://openreview.net/forum?id=H1aIuk-RW} {Active learning for convolutional neural networks: A core-set approach}.
\newblock In \emph{International Conference on Learning Representations}.

\bibitem[{Settles(2009)}]{settles2009active}
Burr Settles. 2009.
\newblock \href {https://burrsettles.com/pub/settles.activelearning.pdf} {Active learning literature survey}.
\newblock Computer Sciences Technical Report 1648, University of Wisconsin--Madison.

\bibitem[{Shao et~al.(2024)Shao, Chen, Zhang, Xu, Zhao, Li, Zhang, Gao, Qiao, and Luo}]{shao2024omniquant}
Wenqi Shao, Mengzhao Chen, Zhaoyang Zhang, Peng Xu, Lirui Zhao, Zhiqian Li, Kaipeng Zhang, Peng Gao, Yu~Qiao, and Ping Luo. 2024.
\newblock \href {https://openreview.net/forum?id=8Wuvhh0LYW} {{OmniQuant}: Omnidirectionally calibrated quantization for large language models}.
\newblock In \emph{International Conference on Learning Representations}.

\bibitem[{Shao et~al.(2026)Shao, Zhang, Cui, Wang, Jiang, Ye, Wang, Du, Fu, Yang et~al.}]{shao2026decodeshare}
Zishan Shao, Lixun Zhang, Kangning Cui, Yixiao Wang, Ting Jiang, Hancheng Ye, Qinsi Wang, Zhixu Du, Yuzhe Fu, Fan Yang, and 1 others. 2026.
\newblock Decodeshare: Tracing the shared subspace of llm decode-time decisions.
\newblock In \emph{Proceedings of the International Conference on Machine Learning}.

\bibitem[{Talmor et~al.(2019)Talmor, Herzig, Lourie, and Berant}]{talmor-etal-2019-commonsenseqa}
Alon Talmor, Jonathan Herzig, Nicholas Lourie, and Jonathan Berant. 2019.
\newblock \href {https://doi.org/10.18653/v1/N19-1421} {{C}ommonsense{QA}: A question answering challenge targeting commonsense knowledge}.
\newblock In \emph{Proceedings of the 2019 Conference of the North {A}merican Chapter of the Association for Computational Linguistics: Human Language Technologies, Volume 1 (Long and Short Papers)}, pages 4149--4158, Minneapolis, Minnesota. Association for Computational Linguistics.

\bibitem[{Tseng et~al.(2024)Tseng, Chee, Sun, Kuleshov, and De~Sa}]{tseng2024quipsharp}
Albert Tseng, Jerry Chee, Qingyao Sun, Volodymyr Kuleshov, and Christopher De~Sa. 2024.
\newblock \href {https://proceedings.mlr.press/v235/tseng24a.html} {{QuIP\#}: Even better {LLM} quantization with hadamard incoherence and lattice codebooks}.
\newblock In \emph{Proceedings of the 41st International Conference on Machine Learning}, volume 235 of \emph{Proceedings of Machine Learning Research}, pages 48630--48656. PMLR.

\bibitem[{Williams and Aletras(2024)}]{williams-aletras-2024-impact}
Miles Williams and Nikolaos Aletras. 2024.
\newblock \href {https://doi.org/10.18653/v1/2024.acl-long.544} {On the impact of calibration data in post-training quantization and pruning}.
\newblock In \emph{Proceedings of the 62nd Annual Meeting of the Association for Computational Linguistics (Volume 1: Long Papers)}, pages 10100--10118, Bangkok, Thailand. Association for Computational Linguistics.

\bibitem[{Williams et~al.(2025)Williams, Chrysostomou, and Aletras}]{williams2025self}
Miles Williams, George Chrysostomou, and Nikolaos Aletras. 2025.
\newblock Self-calibration for language model quantization and pruning.
\newblock In \emph{Proceedings of the 2025 Conference of the Nations of the Americas Chapter of the Association for Computational Linguistics: Human Language Technologies (Volume 1: Long Papers)}, pages 10149--10167.

\bibitem[{Wu et~al.(2026)Wu, Shao, Cui, Kim, Wang, Ye, Zhuo, and Chen}]{wu2026flashsvd}
Wenhao Wu, Zishan Shao, Kangning Cui, Jinhee Kim, Yixiao Wang, Hancheng Ye, Danyang Zhuo, and Yiran Chen. 2026.
\newblock Flashsvd v1. 5: Making low-rank transformers inference actually fast.
\newblock \emph{arXiv preprint arXiv:2605.08314}.

\bibitem[{Xiao et~al.(2023)Xiao, Lin, Seznec, Wu, Demouth, and Han}]{xiao2023smoothquant}
Guangxuan Xiao, Ji~Lin, Mickael Seznec, Hao Wu, Julien Demouth, and Song Han. 2023.
\newblock \href {https://proceedings.mlr.press/v202/xiao23c.html} {{SmoothQuant}: Accurate and efficient post-training quantization for large language models}.
\newblock In \emph{Proceedings of the 40th International Conference on Machine Learning}.

\bibitem[{Xiao et~al.(2022)Xiao, Liang, Bhatt, Neiswanger, Salakhutdinov, and Morency}]{xiao-etal-2022-uncertainty}
Yuxin Xiao, Paul~Pu Liang, Umang Bhatt, Willie Neiswanger, Ruslan Salakhutdinov, and Louis-Philippe Morency. 2022.
\newblock \href {https://doi.org/10.18653/v1/2022.findings-emnlp.538} {Uncertainty quantification with pre-trained language models: A large-scale empirical analysis}.
\newblock In \emph{Findings of the Association for Computational Linguistics: EMNLP 2022}, pages 7273--7284, Abu Dhabi, United Arab Emirates. Association for Computational Linguistics.

\bibitem[{Yao et~al.(2022)Yao, Yazdani~Aminabadi, Zhang, Wu, Li, and He}]{yao2022zeroquant}
Zhewei Yao, Reza Yazdani~Aminabadi, Minjia Zhang, Xiaoxia Wu, Conglong Li, and Yuxiong He. 2022.
\newblock \href {https://proceedings.neurips.cc/paper_files/paper/2022/hash/adf7fa39d65e2983d724ff7da57f00ac-Abstract-Conference.html} {{ZeroQuant}: Efficient and affordable post-training quantization for large-scale transformers}.
\newblock In \emph{Advances in Neural Information Processing Systems}, volume~35.

\bibitem[{Zadrozny and Elkan(2002)}]{zadrozny2002transforming}
Bianca Zadrozny and Charles Elkan. 2002.
\newblock \href {https://doi.org/10.1145/775047.775151} {Transforming classifier scores into accurate multiclass probability estimates}.
\newblock In \emph{Proceedings of the Eighth ACM SIGKDD International Conference on Knowledge Discovery and Data Mining}, pages 694--699.

\bibitem[{Zellers et~al.(2019)Zellers, Holtzman, Bisk, Farhadi, and Choi}]{zellers-etal-2019-hellaswag}
Rowan Zellers, Ari Holtzman, Yonatan Bisk, Ali Farhadi, and Yejin Choi. 2019.
\newblock \href {https://doi.org/10.18653/v1/P19-1472} {{H}ella{S}wag: Can a machine really finish your sentence?}
\newblock In \emph{Proceedings of the 57th Annual Meeting of the Association for Computational Linguistics}, pages 4791--4800, Florence, Italy. Association for Computational Linguistics.

\bibitem[{Zhao et~al.(2021)Zhao, Wallace, Feng, Klein, and Singh}]{zhao2021calibrate}
Tony~Z. Zhao, Eric Wallace, Shi Feng, Dan Klein, and Sameer Singh. 2021.
\newblock \href {https://proceedings.mlr.press/v139/zhao21c.html} {Calibrate before use: Improving few-shot performance of language models}.
\newblock In \emph{Proceedings of the 38th International Conference on Machine Learning}, pages 12697--12706.

\bibitem[{Zhong et~al.(2025)Zhong, Wang, Chuang, and Zou}]{zhong-etal-2025-quantized}
Mingyu Zhong, Guanchu Wang, Yu-Neng Chuang, and Na~Zou. 2025.
\newblock \href {https://doi.org/10.18653/v1/2025.acl-long.1473} {Quantized can still be calibrated: A unified framework to calibration in quantized large language models}.
\newblock In \emph{Proceedings of the 63rd Annual Meeting of the Association for Computational Linguistics (Volume 1: Long Papers)}, pages 30503--30517, Vienna, Austria. Association for Computational Linguistics.

\end{thebibliography}

\appendix
\section{Additional Experimental Results}
\label{app:full_results}

This appendix collects evidence that supports the main text but is too detailed to include there. It is organized as follows: \S\ref{app:bdry_ablations} reports stability checks and ablations on the answerability-boundary target; \S\ref{app:broad_pres} reports broad-preservation diagnostics; \S\ref{app:posthoc_full} analyzes post-hoc calibration; and \S\ref{app:awq-extension} presents the AWQ quantizer-family extension. Throughout, \S\ref{app:bdry_ablations} uses the answerability-boundary ranking, while \S\ref{app:broad_pres} uses broad \fp-behavior preservation. A method can therefore be strong under one target and only average under another. We keep the most representative evidence in this main appendix, while additional diagnostics for broad preservation, post-hoc calibration, and AWQ transfer are collected in Appendix~\ref{app:additional}.

\subsection{Stability and Ablations on the Answerability-Boundary Target}
\label{app:bdry_ablations}

\paragraph{Stability of the core result.}
Table~\ref{tab:squad_nollama} repeats the \squad answerability analysis after removing the Llama-3.2-1B stress case. The ordering remains consistent: \dpq-s128-r75 stays first, and its \fp agreement rises to 0.9360.

\begin{table*}[t]
\centering
\small
\resizebox{0.7\textwidth}{!}{%
\begin{tabular}{lrrrrr}
\toprule
Method & Rank $\downarrow$ & Agr. $\uparrow$ & Acc$\Delta$ $\downarrow$ & Rate$\Delta$ $\downarrow$ & JSD $\downarrow$ \\
\midrule
\dpq-s128-r75 & 6.26 & 0.9360 & 0.0331 & 0.0094 & 0.0088 \\
\dpq-answerability-only & 7.50 & 0.9250 & 0.0569 & 0.0267 & 0.0100 \\
\dpq-entropy-only & 7.66 & 0.9263 & 0.0891 & 0.0446 & 0.0091 \\
\dpq-boundary-random & 8.26 & 0.9326 & 0.0411 & 0.0183 & 0.0093 \\
\dpq-boundary-only & 8.61 & 0.9316 & 0.0460 & 0.0210 & 0.0099 \\
\gptq-ActivationKCenter & 9.39 & 0.9294 & 0.0606 & 0.0291 & 0.0105 \\
Uncertainty-only & 9.67 & 0.9139 & 0.1197 & 0.0590 & 0.0109 \\
\dpq-s64-r50 & 9.71 & 0.9224 & 0.0769 & 0.0364 & 0.0100 \\
\bottomrule
\end{tabular}}
\caption{\squad answerability-boundary preservation after removing the Llama-3.2-1B stress case.
The main r75 result remains stable.}
\label{tab:squad_nollama}
\end{table*}

\paragraph{Answerability distribution diagnostics.}
Figure~\ref{fig:appendix-answerability-dist} provides the same \squad answerability-distribution diagnostic for Qwen2.5-7B and Mistral-7B-v0.3, comparing full precision, GPTQ-WikiText, and \dpq-r75 via the ECDF of $p(\mathrm{answerable})$.
Together with Figure~\ref{fig:llama-answerability-dist}, these plots show that the effect is not a single-model artifact: \dpq-r75 generally reduces quantization-induced answerability-distribution drift, while the direction and magnitude of the residual drift are model-dependent.

\begin{figure*}[htbp]
\centering
\includegraphics[width=\textwidth]{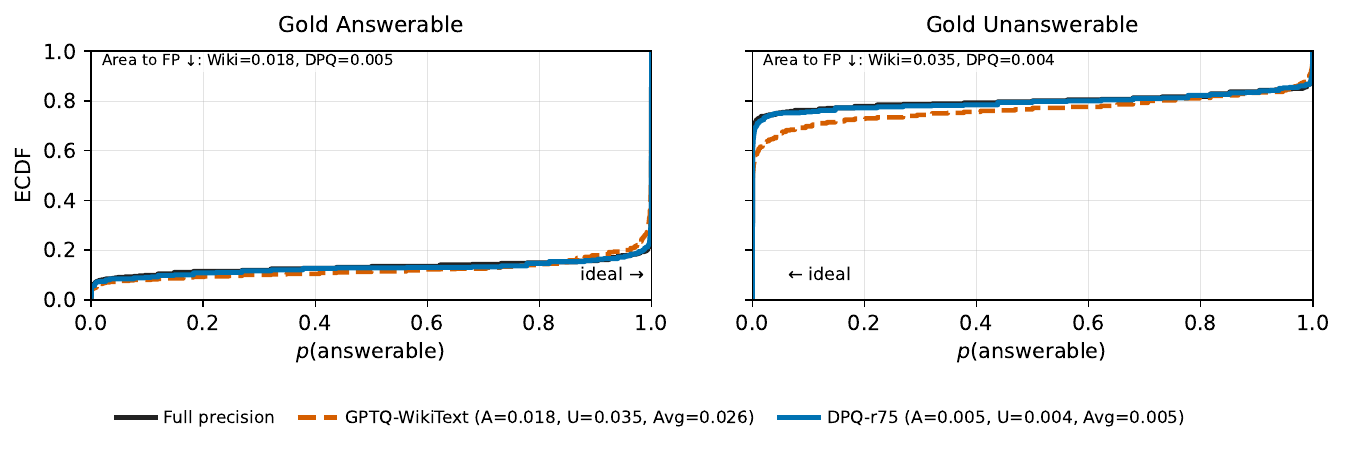}

\vspace{1.0em}

\includegraphics[width=\textwidth]{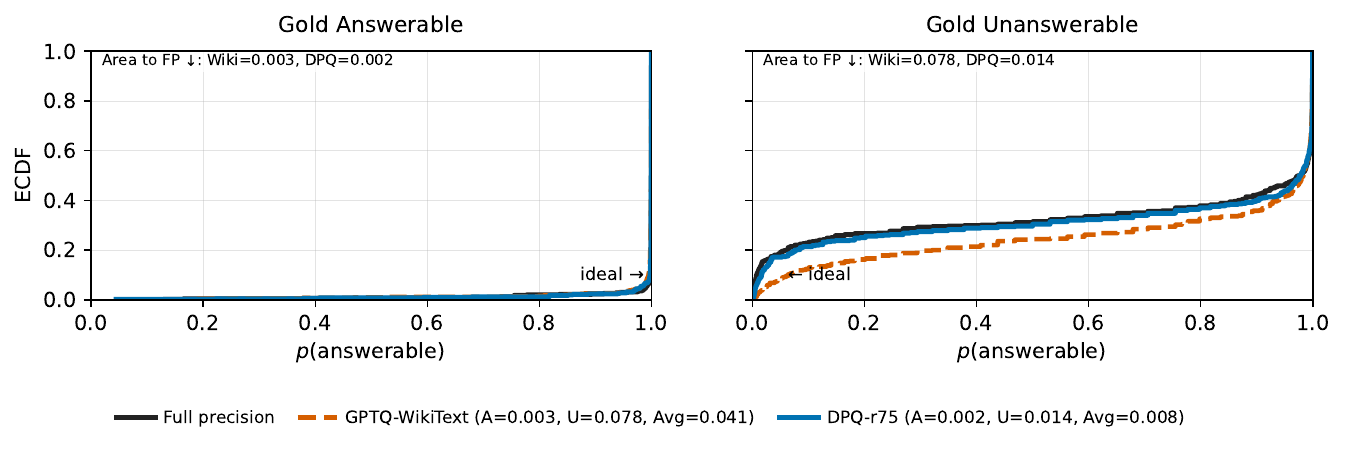}
\caption{
Additional \squad answerability ECDFs for Qwen2.5-7B (top) and Mistral-7B-v0.3 (bottom). Lower area-to-\fp indicates closer preservation of \fp answerability behavior.
}
\label{fig:appendix-answerability-dist}
\end{figure*}

\paragraph{Ablations on the boundary target.}
Table~\ref{tab:bdry_ablations} reports the full ablation grid grouped by axis. The mixture ratio is non-monotonic: both r0 and r100 trail r75, indicating that boundary mass should be high but not exclusive. The size sweep shows that composition of the 128 examples matters more than scaling to 64 or 256. Boundary-only and boundary-random fall below r75, supporting the role of generic anchors; each single-signal component is informative but none matches the mixed recipe. Negative controls (HighNLL-QA, LowDoubt-QA) trail r75 by $\geq 0.09$ on boundary-accuracy deviation, distinguishing difficulty from doubt (Appendix~\ref{app:hard_example_theory}). Strong data-selection baselines---Activation k-center, Uncertainty-only, TaskRandom, RandomQA, and self-calibration---improve over generic WikiText/C4 calibration but remain below \dpq-r75 on the boundary-specific aggregate.

\begin{table*}[t]
\centering
\small
\resizebox{0.8\textwidth}{!}{%
\begin{tabular}{llrrrrr}
\toprule
Group & Method & Rank $\downarrow$ & Agr. $\uparrow$ & Acc$\Delta$ $\downarrow$ & Rate$\Delta$ $\downarrow$ & JSD $\downarrow$ \\
\midrule
Reference & \dpq-s128-r75 & 6.20 & 0.8495 & 0.2125 & 0.1000 & 0.0158 \\
\midrule
\multirow{6}{*}{Ratio / size}
 & \dpq-s64-r50    & 10.18 & 0.8245 & 0.2800 & 0.1382 & 0.0197 \\
 & \dpq-s256-r50   & 11.18 & 0.8240 & 0.2870 & 0.1425 & 0.0235 \\
 & \dpq-s128-r0    & 10.31 & 0.8097 & 0.3205 & 0.1598 & 0.0220 \\
 & \dpq-s128-r25   & 12.06 & 0.8106 & 0.3277 & 0.1639 & 0.0214 \\
 & \dpq-s128-r50   & 13.12 & 0.8140 & 0.3210 & 0.1602 & 0.0258 \\
 & \dpq-s128-r100  & 12.44 & 0.8080 & 0.3340 & 0.1655 & 0.0215 \\
\midrule
\multirow{2}{*}{Boundary mix}
 & \dpq-boundary-random & 8.15 & 0.8413 & 0.2310 & 0.1135 & 0.0174 \\
 & \dpq-boundary-only   & 9.06 & 0.8335 & 0.2520 & 0.1242 & 0.0197 \\
\midrule
\multirow{3}{*}{Single signal}
 & \dpq-answerability-only & 8.18 & 0.8286 & 0.2592 & 0.1281 & 0.0205 \\
 & \dpq-entropy-only       & 8.68 & 0.8270 & 0.2940 & 0.1470 & 0.0227 \\
 & \dpq-confidence-only    & 11.30 & 0.8033 & 0.3460 & 0.1725 & 0.0221 \\
\midrule
\multirow{2}{*}{Neg.\ control}
 & HighNLL-QA  & 11.11 & 0.8191 & 0.3047 & 0.1514 & 0.0201 \\
 & LowDoubt-QA & 12.09 & 0.8177 & 0.3110 & 0.1532 & 0.0202 \\
\midrule
\multirow{7}{*}{Data selection}
 & \gptq-ActivationKCenter & 9.39  & 0.8360 & 0.2530 & 0.1255 & 0.0187 \\
 & Uncertainty-only        & 9.36  & 0.8261 & 0.2937 & 0.1461 & 0.0181 \\
 & \gptq-TaskRandom        & 10.81 & 0.8240 & 0.2920 & 0.1445 & 0.0216 \\
 & \gptq-RandomQA          & 11.09 & 0.8266 & 0.2733 & 0.1351 & 0.0231 \\
 & \gptq-SelfCalib         & 13.24 & 0.8129 & 0.3037 & 0.1509 & 0.0268 \\
 & \gptq-C4                & 15.78 & 0.7594 & 0.4107 & 0.2034 & 0.0342 \\
 & \gptq-WikiText          & 17.29 & 0.7446 & 0.4553 & 0.2271 & 0.0387 \\
\bottomrule
\end{tabular}}
\caption{Boundary-target ablation grid on \squad answerability. Methods are grouped by ablation axis; columns follow Table~\ref{tab:squad}. Acc$\Delta$ and Rate$\Delta$ denote boundary-accuracy and answerability-rate deviation, respectively.}
\label{tab:bdry_ablations}
\end{table*}

\subsection{Broad-Preservation Diagnostics}
\label{app:broad_pres}
When the target shifts from \squad answerability boundaries to broad \fp-behavior preservation, the boundary-heavy \dpq-r75 mixture is not expected to dominate. Table~\ref{tab:extra6_full} expands Table~\ref{tab:extra} on the six answerable extra-\mcqa datasets and is the most representative broad-target leaderboard supporting the target-dependent claim of Section~\ref{sec:mcqa_result}. Confidence-only, TaskRandom, HighNLL-QA, entropy-only, and uncertainty-only all sit above \dpq-r75; \bnb has the strongest \fp agreement and JSD but is not the best average rank, showing that preserving one metric family is insufficient under the full multi-metric objective. Additional diagnostics on calibration-gap prevalence and correct-vs-wrong confidence decomposition are reported in Appendix~\ref{app:additional_broad}; they further show that quantization can change confidence and margin behavior even when top-1 accuracy gives an incomplete picture.

\begin{table*}[t]
\centering
\small
\resizebox{0.85\textwidth}{!}{%
\begin{tabular}{@{}llccccc@{}}
\toprule
Method & Family & Rank $\downarrow$ & Top-5 $\uparrow$ & Agr. $\uparrow$ & JSD $\downarrow$ & Margin$\Delta$ $\downarrow$ \\
\midrule
\dpq{}-confidence-only & \dpq{} component & 8.96 & 7 & 0.8357 & 0.0284 & 0.1447 \\
\gptq{}-TaskRandom & Selection baseline & 9.12 & 11 & 0.8361 & 0.0276 & 0.1452 \\
HighNLL-QA & Negative control & 9.18 & 6 & 0.8327 & 0.0294 & 0.1478 \\
\dpq{}-entropy-only & \dpq{} component & 9.30 & 10 & 0.8324 & 0.0296 & 0.1473 \\
Uncertainty-only & Selection baseline & 9.31 & 4 & 0.8318 & 0.0297 & 0.1471 \\
LowDoubt-QA & Negative control & 9.47 & 9 & 0.8328 & 0.0286 & 0.1472 \\
\gptq{}-ActivationKCenter & Selection baseline & 9.53 & 12 & 0.8349 & 0.0299 & 0.1483 \\
\dpq{}-s128-r50 & \dpq{} mixed & 9.92 & 4 & 0.8322 & 0.0293 & 0.1454 \\
\bnb{} & Quantizer baseline & 10.53 & 12 & 0.8637 & 0.0241 & 0.1320 \\
\dpq{}-boundary-random & \dpq{} mixed & 10.66 & 8 & 0.8279 & 0.0311 & 0.1513 \\
\dpq{}-s128-r75 & \dpq{} mixed & 11.88 & 3 & 0.8180 & 0.0324 & 0.1562 \\
\bottomrule
\end{tabular}}
\caption{
Expanded broad-\mcqa preservation on the six answerable datasets.
The table adds method family and margin drift to the main broad-\mcqa summary in Table~\ref{tab:extra}.
}
\label{tab:extra6_full}
\end{table*}

\subsection{Post-Hoc Calibration Analysis}
\label{app:posthoc_full}

\paragraph{Post-hoc on both suites.}
Table~\ref{tab:posthoc_old_extra} extends the main post-hoc analysis (Table~\ref{tab:posthoc}) to both the old-core and extra-\mcqa suites, providing the most representative summary of how each post-hoc family behaves on \fp-behavior metrics. Adaptive temperature leaves accuracy and \fp agreement unchanged but increases JSD and margin drift, consistent with Proposition~\ref{prop:temp_argmax}. Flexible calibrators (vector, matrix, Dirichlet, isotonic) raise accuracy --- particularly on old-core --- while reducing \fp agreement and increasing distributional drift, consistent with Proposition~\ref{prop:posthoc_conflict}. The two stages thus serve different objectives: post-hoc calibration reshapes scores, whereas pre-quantization data selection controls which behavior is preserved in the first place. Appendix~\ref{app:additional_posthoc} reports an optimally fitted temperature-scaling check, which supports the same complementary-not-substitute conclusion.

\begin{table*}[t]
\centering
\small
\resizebox{0.7\textwidth}{!}{%
\begin{tabular}{@{}llccccc@{}}
\toprule
Suite & Method & Rows & Acc. $\uparrow$ & Agr. $\uparrow$ & JSD $\downarrow$ & Margin$\Delta$ $\downarrow$ \\
\midrule
\multirow{7}{*}{Old-core}
& Base & 72 & 0.6196 & 0.7906 & 0.0364 & 0.1609 \\
& Adaptive temp. & 72 & 0.6196 & 0.7906 & 0.0605 & 0.2892 \\
& Option bias & 72 & 0.5991 & 0.7811 & 0.0410 & 0.1876 \\
& Vector & 72 & 0.7985 & 0.6875 & 0.1360 & 0.3099 \\
& Matrix & 72 & 0.7925 & 0.6796 & 0.1404 & 0.3146 \\
& Dirichlet & 72 & 0.7929 & 0.6817 & 0.1403 & 0.3162 \\
& Isotonic & 72 & 0.7983 & 0.6883 & 0.1398 & 0.2990 \\
\midrule
\multirow{7}{*}{Extra-\mcqa}
& Base & 1056 & 0.6833 & 0.8245 & 0.0320 & 0.1533 \\
& Adaptive temp. & 1056 & 0.6833 & 0.8245 & 0.0445 & 0.2385 \\
& Option bias & 1056 & 0.6985 & 0.8147 & 0.0328 & 0.1591 \\
& Vector & 1056 & 0.7029 & 0.7773 & 0.0519 & 0.2344 \\
& Matrix & 1056 & 0.6992 & 0.7723 & 0.0550 & 0.2380 \\
& Dirichlet & 1056 & 0.6989 & 0.7740 & 0.0548 & 0.2400 \\
& Isotonic & 1056 & 0.6991 & 0.7710 & 0.0561 & 0.2329 \\
\bottomrule
\end{tabular}}
\caption{
Post-hoc family summary on old-core and extra-\mcqa suites.
Flexible calibrators can improve accuracy but often reduce agreement with \fp and increase distributional drift.
}
\label{tab:posthoc_old_extra}
\end{table*}

\subsection{AWQ Quantizer-Family Extension}
\label{app:awq-extension}

The AWQ extension covers all eight models, three old-core datasets, and four AWQ calibration recipes (WikiText, RandomQA, DPQ-r75, DPQ-r50), totaling $8\times 3\times 4=96$ evaluations, to test whether the calibration-data signal observed under GPTQ transfers to a different quantizer family. On AWQ, score-based metrics (Acc, ECE, NLL, Brier) do not always favor the same recipe as \fp-behavior metrics (\fp agreement, JSD-to-\fp), so we report aggregate metric averages and rank-based summaries separately in Appendix~\ref{app:additional_awq}. Aggregate summaries in Appendix~\ref{app:additional_awq} show the same mixed pattern: generic calibration can win score-based metrics, while \dpq variants are stronger on several \fp-behavior and rank-based summaries. On the 7B/8B subset (Qwen2.5-7B, Llama-3.1-8B, Mistral-7B-v0.3), \dpq-r75 additionally leads on accuracy, Brier, \fp agreement, and the rank-based summaries. The most informative single view is the per-dataset breakdown below.

\paragraph{AWQ per-dataset breakdown.}
Table~\ref{tab:awq-per-dataset} reports the full 3-dataset $\times$ 4-recipe AWQ grid. The transfer to AWQ is dataset-dependent. On ARC-Challenge, \dpq-r50 has the best accuracy and \fp agreement and \dpq-r75 the best Brier score, so \dpq variants dominate. On \squad answerability the picture inverts: WikiText takes the best accuracy, ECE, NLL, and Brier, while \dpq variants (r50 marginally above r75) retain only the \fp-agreement lead. On TruthfulQA, WikiText is strongest on every metric reported here. This pattern is consistent with the main message: the calibration-data signal is real under AWQ, but the optimal recipe is quantizer- and target-dependent rather than fixed at r75.

\begin{table*}[t]
\centering
\small
\resizebox{0.85\linewidth}{!}{%
\begin{tabular}{llrrrrrr}
\toprule
Dataset & AWQ recipe & $n$ & Acc. $\uparrow$ & ECE $\downarrow$ & NLL $\downarrow$ & Brier $\downarrow$ & Agr. $\uparrow$ \\
\midrule
\multirow{4}{*}{ARC-Challenge}
& WikiText & 8 & 0.6994 & 0.1185 & 0.8645 & 0.4133 & 0.8516 \\
& RandomQA & 8 & 0.7082 & \textbf{0.0920} & \textbf{0.8191} & 0.3974 & 0.8637 \\
& DPQ-r75 & 8 & 0.7153 & 0.1040 & 0.8360 & \textbf{0.3952} & 0.8662 \\
& DPQ-r50 & 8 & \textbf{0.7174} & 0.1136 & 0.8373 & 0.4008 & \textbf{0.8792} \\
\midrule
\multirow{4}{*}{\squad answerability}
& WikiText & 8 & \textbf{0.6474} & \textbf{0.2122} & \textbf{1.0830} & \textbf{0.5330} & 0.8483 \\
& RandomQA & 8 & 0.6429 & 0.2237 & 1.2086 & 0.5472 & 0.8398 \\
& DPQ-r75 & 8 & 0.6301 & 0.2501 & 1.2885 & 0.5687 & 0.8545 \\
& DPQ-r50 & 8 & 0.6324 & 0.2502 & 1.3269 & 0.5830 & \textbf{0.8558} \\
\midrule
\multirow{4}{*}{TruthfulQA}
& WikiText & 8 & \textbf{0.5332} & \textbf{0.2366} & \textbf{1.7355} & \textbf{0.6887} & \textbf{0.8040} \\
& RandomQA & 8 & 0.5300 & 0.2460 & 1.7858 & 0.6985 & 0.7945 \\
& DPQ-r75 & 8 & 0.4991 & 0.2513 & 1.7787 & 0.7111 & 0.7930 \\
& DPQ-r50 & 8 & 0.4841 & 0.2681 & 1.8821 & 0.7442 & 0.7971 \\
\bottomrule
\end{tabular}}
\caption{AWQ per-dataset $\times$ recipe breakdown averaged over the 8 models. Best values for each dataset and metric are bolded. Agr. denotes agreement with the full-precision model. The preferred recipe varies by dataset and metric, showing that AWQ transfer is target- and metric-dependent rather than fixed to a single calibration recipe.
}
\label{tab:awq-per-dataset}
\end{table*}

\section{Additional Diagnostic Analyses}
\label{app:additional}

This appendix collects additional diagnostic analyses that further decompose or stress-test the results summarized in Appendix~\ref{app:full_results}. These checks provide secondary decompositions and stress tests, while Appendix~\ref{app:full_results} keeps the primary evidence focused on the main target-specific claims.

\subsection{Additional Broad-Preservation Diagnostics}
\label{app:additional_broad}

This subsection provides two compact diagnostics supporting the claim that quantization changes score behavior in ways not captured by top-1 accuracy alone. Table~\ref{tab:gap_full} extends the main degradation summary with confidence shift and margin drift. Table~\ref{tab:correct_wrong} further decomposes confidence on correct and wrong predictions, showing that average calibration metrics can hide changes in the separation between reliable and unreliable outputs.

\paragraph{Calibration-gap prevalence.}
Table~\ref{tab:gap_full} extends Table~\ref{tab:gap} with confidence shift and margin drift. Quantization frequently worsens not only accuracy but also calibration-sensitive and drift metrics. The last two columns highlight that generic-text calibration (WikiText, C4) induces the largest mean confidence and margin shifts, and no row eliminates degradation across all columns.

\begin{table*}[t]
\centering
\small
\resizebox{0.75\textwidth}{!}{%
\begin{tabular}{@{}lccccc cc@{}}
\toprule
\multirow{2}{*}{Method} & \multirow{2}{*}{$n$}
& \multicolumn{4}{c}{Worse than \fp{} (\%, $\downarrow$)}
& \multicolumn{2}{c}{Mean drift ($\downarrow$)} \\
\cmidrule(lr){3-6}\cmidrule(l){7-8}
& & Acc. & ECE & NLL & Brier & Conf. shift & Margin$\Delta$ \\
\midrule
\bnb{} & 72 & 84.7 & 66.7 & 84.7 & 81.9 & 0.0190 & 0.1376 \\
\gptq{}-RandomQA & 72 & 83.3 & 61.1 & 76.4 & 86.1 & 0.0373 & 0.1500 \\
\gptq{}-WikiText & 72 & 76.4 & 58.3 & 63.9 & 83.3 & 0.0583 & 0.1901 \\
\gptq{}-C4 & 72 & 83.3 & 65.3 & 68.1 & 77.8 & 0.0493 & 0.1812 \\
\dpq{}-s128-r75 & 72 & 80.6 & 62.5 & 70.8 & 83.3 & 0.0363 & 0.1543 \\
\dpq{}-s128-r50 & 72 & 79.2 & 58.3 & 65.3 & 80.6 & 0.0321 & 0.1474 \\
\dpq{}-confidence-only & 72 & 80.6 & 62.5 & 66.7 & 81.9 & 0.0323 & 0.1468 \\
\dpq{}-entropy-only & 72 & 76.4 & 73.6 & 75.0 & 86.1 & 0.0322 & 0.1445 \\
Uncertainty-only & 72 & 77.8 & 59.7 & 73.6 & 80.6 & 0.0331 & 0.1473 \\
\gptq{}-TaskRandom & 72 & 84.7 & 65.3 & 73.6 & 83.3 & 0.0320 & 0.1448 \\
\bottomrule
\end{tabular}}
\caption{
Calibration-gap prevalence on all nine datasets.
The first four metric columns follow Table~\ref{tab:gap}; the last two columns report mean confidence shift and margin drift relative to \fp.
}
\label{tab:gap_full}
\end{table*}

\paragraph{Correct-vs-wrong confidence decomposition.}
Table~\ref{tab:correct_wrong} splits confidence into a gap on correct predictions and overconfidence on wrong ones. Quantized methods often reduce confidence on correct predictions while keeping wrong-prediction confidence high on extra-\mcqa, so a method may appear only mildly shifted in average confidence while degrading the separation between reliable and unreliable outputs.

\begin{table*}[t]
\centering
\small
\resizebox{0.8\textwidth}{!}{%
\begin{tabular}{llrrrr}
\toprule
Suite & Method & Acc. & Conf. correct & Correct gap & Conf. wrong \\
\midrule
\multirow{7}{*}{Old-core}
& \fp & 0.6483 & 0.8358 & 0.1642 & 0.7225 \\
& \dpq-s128-r75 & 0.6136 & 0.8054 & 0.1946 & 0.6953 \\
& \dpq-confidence-only & 0.6338 & 0.8106 & 0.1894 & 0.6996 \\
& \dpq-entropy-only & 0.6418 & 0.8153 & 0.1847 & 0.7071 \\
& \gptq-RandomQA & 0.6316 & 0.8089 & 0.1911 & 0.7042 \\
& \gptq-TaskRandom & 0.6296 & 0.8087 & 0.1913 & 0.6999 \\
& \bnb & 0.6239 & 0.8353 & 0.1647 & 0.7233 \\
\midrule
\multirow{8}{*}{Extra-\mcqa}
& \fp & 0.7232 & 0.8424 & 0.1576 & 0.7070 \\
& \dpq-s128-r75 & 0.6892 & 0.8134 & 0.1866 & 0.6929 \\
& \dpq-s128-r50 & 0.6898 & 0.8166 & 0.1834 & 0.6929 \\
& \dpq-confidence-only & 0.6893 & 0.8192 & 0.1808 & 0.6946 \\
& \dpq-entropy-only & 0.6861 & 0.8166 & 0.1834 & 0.6979 \\
& \gptq-RandomQA & 0.6884 & 0.8063 & 0.1937 & 0.6817 \\
& \gptq-TaskRandom & 0.6861 & 0.8116 & 0.1884 & 0.6884 \\
& \bnb & 0.6945 & 0.8348 & 0.1652 & 0.7102 \\
\bottomrule
\end{tabular}}
\caption{Correct-vs-wrong confidence decomposition. ``Correct gap'' is $1-\mathbb{E}[\mathrm{conf}\mid \mathrm{correct}]$; ``Conf.\ wrong'' is $\mathbb{E}[\mathrm{conf}\mid \mathrm{wrong}]$, which directly measures overconfidence on errors.}
\label{tab:correct_wrong}
\end{table*}

\subsection{Additional Post-Hoc Calibration Checks}
\label{app:additional_posthoc}

This subsection adds an optimally fitted temperature-scaling check to complement Table~\ref{tab:posthoc_old_extra}. The goal is to test whether the temperature result in the main post-hoc analysis is simply due to an under-tuned scalar temperature.

\paragraph{Optimally fitted temperature.}
Table~\ref{tab:temperature-aggregate} fits one optimal temperature per model--dataset setting. Temperature scaling improves ECE and NLL deltas relative to \fp, but it leaves accuracy, \fp agreement, and boundary decisions unchanged by construction. This supports the view that temperature scaling can polish proper scores but cannot repair decision-surface drift.

\begin{table}[t]
\centering
\small
\resizebox{\columnwidth}{!}{%
\begin{tabular}{llrrr}
\toprule
Suite & Method & $n$ & ECE $\Delta$ pre$\to$post & NLL $\Delta$ pre$\to$post \\
\midrule
\multirow{4}{*}{Old-core}
& \gptq{}-WikiText & 24 & $-0.009 \to -0.091$ & $-0.050 \to -0.308$ \\
& \gptq{}-RandomQA & 24 & $-0.003 \to -0.097$ & $+0.013 \to -0.314$ \\
& \dpq{}-r75 & 24 & $+0.002 \to -0.102$ & $+0.054 \to -0.296$ \\
& \bnb{} & 24 & $+0.012 \to -0.091$ & $+0.099 \to -0.313$ \\
\midrule
\multirow{4}{*}{Extra-\mcqa}
& \gptq{}-WikiText & 48 & $+0.015 \to -0.048$ & $+0.091 \to -0.037$ \\
& \gptq{}-RandomQA & 48 & $+0.002 \to -0.055$ & $+0.046 \to -0.066$ \\
& \dpq{}-r75 & 48 & $+0.004 \to -0.056$ & $+0.045 \to -0.071$ \\
& \bnb{} & 48 & $+0.015 \to -0.050$ & $+0.078 \to -0.076$ \\
\bottomrule
\end{tabular}}
\caption{
Effect of fitting one optimal temperature per model--dataset setting. Temperature improves proper-score deltas but does not change accuracy, \fp agreement, or boundary decisions.
}
\label{tab:temperature-aggregate}
\end{table}

\subsection{AWQ Aggregate Summaries}
\label{app:additional_awq}

Tables~\ref{tab:awq-metric} and~\ref{tab:awq-rank} report aggregate AWQ metric averages and rank-based summaries, complementing the per-dataset breakdown in Table~\ref{tab:awq-per-dataset}. Across all 8 models, WikiText or RandomQA win several score-based metrics while DPQ variants are stronger on \fp-behavior metrics; the rank aggregate gives \dpq-r75 the best mean rank and top-2 frequency. The 7B/8B subset is cleaner: \dpq-r75 leads on accuracy, Brier, \fp agreement, mean rank, and top-1/top-2 frequencies, with RandomQA only marginally better on ECE and JSD-to-\fp. These aggregate results are consistent with stronger transfer on the 7B/8B-scale models, while smaller models introduce more variability.

\begin{table*}[t]
\centering
\small
\resizebox{0.8\linewidth}{!}{
\begin{tabular}{llrrrrrr}
\toprule
Scope & AWQ method & Acc. $\uparrow$ & ECE $\downarrow$ & NLL $\downarrow$ & Brier $\downarrow$ & Agr. $\uparrow$ & JSD $\downarrow$ \\
\midrule
\multirow{4}{*}{All 8 models}
 & WikiText & 0.6267 & 0.1891 & \textbf{1.2277} & \textbf{0.5450} & 0.8346 & 0.0325 \\
 & RandomQA & \textbf{0.6270} & \textbf{0.1872} & 1.2712 & 0.5477 & 0.8327 & 0.0336 \\
 & DPQ-r75  & 0.6148 & 0.2018 & 1.3011 & 0.5583 & 0.8379 & \textbf{0.0320} \\
 & DPQ-r50  & 0.6113 & 0.2106 & 1.3488 & 0.5760 & \textbf{0.8440} & 0.0332 \\
\midrule
\multirow{4}{*}{7B/8B subset}
 & WikiText & 0.6987 & 0.1857 & \textbf{1.2220} & 0.4784 & 0.8950 & 0.0262 \\
 & RandomQA & 0.7105 & \textbf{0.1795} & 1.2368 & 0.4618 & 0.9116 & \textbf{0.0210} \\
 & DPQ-r75  & \textbf{0.7160} & 0.1802 & 1.2386 & \textbf{0.4558} & \textbf{0.9138} & 0.0213 \\
 & DPQ-r50  & 0.7083 & 0.1842 & 1.2590 & 0.4681 & 0.9034 & 0.0224 \\
\bottomrule
\end{tabular}
}
\caption{AWQ metric averages over old-core datasets, split by model scope.
The 7B/8B subset is Qwen2.5-7B, Llama-3.1-8B, and Mistral-7B-v0.3.
Agr. denotes agreement with the full-precision model.}
\label{tab:awq-metric}
\end{table*}

\begin{table}[t]
\centering
\small
\resizebox{\columnwidth}{!}{
\begin{tabular}{llrrr}
\toprule
Scope & AWQ method & Mean rank $\downarrow$ & Top-1 freq. $\uparrow$ & Top-2 freq. $\uparrow$ \\
\midrule
\multirow{4}{*}{All 8 models}
 & DPQ-r75  & \textbf{2.25} & \textbf{29.2\%} & \textbf{66.7\%} \\
 & RandomQA & 2.41 & 8.3\%  & 62.5\% \\
 & WikiText & 2.66 & 16.7\% & 45.8\% \\
 & DPQ-r50  & 2.69 & 4.2\%  & 50.0\% \\
\midrule
\multirow{4}{*}{7B/8B subset}
 & DPQ-r75  & \textbf{1.97} & \textbf{42.9\%} & \textbf{73.0\%} \\
 & RandomQA & 2.43 & 27.0\% & 47.6\% \\
 & WikiText & 2.76 & 19.1\% & 46.0\% \\
 & DPQ-r50  & 2.79 & 11.1\% & 36.5\% \\
\bottomrule
\end{tabular}
}
\caption{Rank-based AWQ summary by model scope.}
\label{tab:awq-rank}
\end{table}

\section{Formal Analysis and Proofs}
\label{app:formal}

This section gives a more complete formal account of the mechanism. Recall the goal is to make precise three narrower claims: low-margin examples are fragile under quantization; calibration distributions should match the uncertainty target; and post-hoc calibration is not equivalent to pre-quantization preservation.

\subsection{Notation}

For an input $x$ with $K$ candidate options, let $s_0(x)\in\mathbb{R}^{K}$ be the \fp option-score vector and $p_0(x)=\mathrm{softmax}(s_0(x))$. For a quantized model calibrated with set $D$, write $s_D(x)=s_0(x)+\Delta_D(x)$ and $p_D(x)=\mathrm{softmax}(s_D(x))$. Let $i^*(x),j^*(x)$ denote the \fp top option and runner-up. The logit margin is $\gamma(x)=z_{0,i^*}(x)-z_{0,j^*}(x)$, and the probability margin is $m(x)=p_{0,(1)}(x)-p_{0,(2)}(x)$.

\subsection{Boundary Fragility}

\begin{proposition}[Top-two boundary flip condition]
\label{prop:app_flip}
With logit margin $\gamma(x)>0$, the relative ordering of $i^*$ and $j^*$ flips after quantization iff $\Delta_D(x,j^*)-\Delta_D(x,i^*)>\gamma(x)$. If $\|\Delta_D(x)\|_\infty<\gamma(x)/2$, then $i^*$ remains above every other option and the top-1 \fp decision is preserved.
\end{proposition}

\begin{proof}
Write $\delta_k=\Delta_D(x,k)$ for option index $k$. For the top-two pair, $z_D(x,i^*)-z_D(x,j^*)=\gamma(x)+\delta_{i^*}-\delta_{j^*}$, which is negative iff $\delta_{j^*}-\delta_{i^*}>\gamma(x)$. For any $k\neq i^*$, $z_0(x,i^*)-z_0(x,k)\geq\gamma(x)$, so if $\|\Delta_D(x)\|_\infty<\gamma(x)/2$ then $z_D(x,i^*)-z_D(x,k)>0$.
\end{proof}

\begin{corollary}[Flip mass is controlled by the margin distribution]
\label{cor:margin_mass}
Assume $\|\Delta_D(x)\|_\infty\leq \eta$ for all $x$ in a test distribution $T$. Then
$\Pr_{x\sim T}[\arg\max z_D(x)\neq \arg\max z_0(x)] \leq \Pr_{x\sim T}[\gamma(x)\leq 2\eta]$.
\end{corollary}

\subsection{Calibration Distribution and Target Mismatch}
\label{app:finite_budget_tradeoff}

Let $q$ denote the distribution over calibration strings used by the quantizer, and $T$ the target evaluation distribution.

\begin{proposition}[Target mismatch bound]
\label{prop:mismatch_app}
For any bounded loss $\ell_D\in[0,1]$, $|R_T(D)-R_q(D)|\leq \mathrm{TV}(T,q)$.
\end{proposition}

\begin{proof}
By the variational characterization, $\sup_{0\leq f\leq 1} |\mathbb{E}_{T}f-\mathbb{E}_{q}f|=\mathrm{TV}(T,q)$. Taking $f=\ell_D$ gives the result.
\end{proof}

\paragraph{Mixture interpretation.}
Write $q_r=r\, q_{\mathrm{bdry}}+(1-r)\, q_{\mathrm{anchor}}$ and $T_\theta=\theta q_{\mathrm{bdry}}+(1-\theta)q_{\mathrm{anchor}}$. If $q_{\mathrm{bdry}}$ and $q_{\mathrm{anchor}}$ have disjoint support, $\mathrm{TV}(T_\theta,q_r)=|\theta-r|$. This is the form used in Proposition~\ref{prop:mismatch} of the main text.

\paragraph{Finite-sample coverage.}
Calibration sets are empirical samples. If random QA calibration samples $n$ examples from a distribution with boundary mass $\rho=\Pr[B]$, then $\Pr[N_B=0]=(1-\rho)^n \leq \exp(-n\rho)$, so $n\geq \log(1/\delta)/\rho$ is needed to see a boundary example with probability $\geq 1-\delta$. When $\rho$ is small, random calibration needs many samples to cover the fragile region. \dpq bypasses this by selecting boundary examples directly using the \fp model.

\subsection{Boundary vs Hard-Example Mining}
\label{app:hard_example_theory}

High NLL does not imply low margin: a binary example with \fp probabilities $(1-\epsilon,\epsilon)$ and gold label option 2 has NLL $-\log\epsilon$ (arbitrarily large) but margin $1-2\epsilon$ (close to 1, far from the boundary). Conversely, an example with \fp probabilities $(1/2+\epsilon,1/2-\epsilon)$ and gold label option 1 has moderate NLL but margin $2\epsilon$ (arbitrarily small). Hard-example mining on gold-label NLL therefore selects a subset distinct from boundary mining; an example can be confidently wrong (high NLL, large margin) without being uncertain.

\subsection{Why Post-hoc Calibration Is Not Equivalent}

\begin{proposition}[Temperature scaling preserves the decision surface]
\label{prop:temp_argmax}
For any $z\in\mathbb{R}^K$ and $T>0$, $\arg\max_i \mathrm{softmax}(z/T)_i = \arg\max_i z_i$. Therefore, scalar temperature scaling cannot repair a top-1 or answerability-boundary flip in the quantized model.
\end{proposition}

\begin{proposition}[Accuracy repair and FP preservation can conflict]
\label{prop:posthoc_conflict}
There exist binary examples for which a post-hoc mapping improves accuracy relative to a quantized model while decreasing agreement with the \fp model.
\end{proposition}

\begin{proof}
Take $p_0=(0.55,0.45)$, $p_Q=(0.45,0.55)$, gold label 2. A post-hoc map producing $p_H=(0.10,0.90)$ lowers NLL on the gold label but has larger JSD from $p_0$. A map restoring agreement with $p_0$ would reduce accuracy on this example. Improving accuracy and preserving the \fp decision surface conflict here.
\end{proof}

\subsection{Correct-vs-Wrong Decomposition}

Let $A$ be the event that the prediction is correct, $c(x)=\max_i p_i(x)$, and $a=\Pr[A]$. With $G_{\mathrm{corr}}=\mathbb{E}[1-c\mid A]$ and $O_{\mathrm{wrong}}=\mathbb{E}[c\mid A^c]$, $\mathbb{E}[c]=a(1-G_{\mathrm{corr}})+(1-a)O_{\mathrm{wrong}}$. Two methods with similar ECE may differ on which component dominates.

\section{Reproducibility Details}
\label{app:repro}

\paragraph{Experimental inventory and completeness.}
Table~\ref{tab:method_inventory} lists the experimental groups. The final merged result files contain zero missing entries for 552 old-core pre-quantization rows and 1104 extra-\mcqa pre-quantization rows.

\begin{table*}[t]
\centering
\small
\resizebox{0.95\linewidth}{!}{%
\begin{tabular}{lll}
\toprule
Group & Purpose & Coverage \\
\midrule
Full precision & unquantized reference & all 9 datasets \\
\bnb & 4-bit NF4 quantizer baseline & all 9 datasets \\
\gptq-WikiText / \gptq-C4 & generic text calibration & WikiText all 9; C4 extended \\
\gptq-RandomQA / TaskRandom & task-formatted random QA & all 9 datasets \\
\dpq-r0/r25/r50/r75/r100 & high-doubt ratio sweep at $s{=}128$ & old-core + extra-\mcqa \\
\dpq-s64/s128/s256-r50 & calibration-size sweep & old-core + extra-\mcqa \\
boundary-only / boundary-random & boundary component controls & old-core + extra-\mcqa \\
confidence / entropy / answerability-only & single-signal DPQ components & old-core + extra-\mcqa \\
low-doubt / high-NLL QA & negative controls & old-core + extra-\mcqa \\
activation k-center / self-calib & strong data-selection baselines & old-core + extra-\mcqa \\
temperature / score-space post-hoc & post-hoc controls & selected core controls \\
AWQ extension & quantizer-family check & old core, 8 models, 4 recipes \\
\bottomrule
\end{tabular}}
\caption{Inventory of experimental groups.}
\label{tab:method_inventory}
\end{table*}


\paragraph{Model sizes and checkpoints.}
Table~\ref{tab:model-identifiers} reports the model families, nominal parameter scales, and public checkpoints used in the experiments. 
All models are instruction-tuned language models in the 0.5B--8B range. 
We report nominal model scales because the exact trainable parameter count can depend on checkpoint metadata and tokenizer/configuration conventions, while the deployment-relevant scale is the advertised checkpoint size.

\begin{table*}[t]
\centering
\small
\resizebox{0.75\linewidth}{!}{
\begin{tabular}{lll}
\toprule
Tag & Nominal scale & Checkpoint \\
\midrule
\texttt{qwen25\_05b} & 0.5B & \texttt{Qwen/Qwen2.5-0.5B-Instruct} \\
\texttt{qwen25\_1p5b} & 1.5B & \texttt{Qwen/Qwen2.5-1.5B-Instruct} \\
\texttt{qwen25\_3b} & 3B & \texttt{Qwen/Qwen2.5-3B-Instruct} \\
\texttt{qwen25\_7b} & 7B & \texttt{Qwen/Qwen2.5-7B-Instruct} \\
\texttt{llama32\_1b\_instruct} & 1B & \texttt{meta-llama/Llama-3.2-1B-Instruct} \\
\texttt{llama32\_3b\_instruct} & 3B & \texttt{meta-llama/Llama-3.2-3B-Instruct} \\
\texttt{llama31\_8b\_instruct} & 8B & \texttt{meta-llama/Llama-3.1-8B-Instruct} \\
\texttt{mistral7b\_v03} & 7B & \texttt{mistralai/Mistral-7B-Instruct-v0.3} \\
\bottomrule
\end{tabular}
}
\caption{Model identifiers and nominal parameter scales.}
\label{tab:model-identifiers}
\end{table*}

\paragraph{Compute infrastructure and budget.}
All experiments were conducted on a single NVIDIA GeForce RTX 5090 GPU with 32GB memory. 
The study uses post-training quantization, calibration-data selection, inference-based evaluation, score-space post-hoc analysis, and AWQ extension experiments; no model was trained from scratch or instruction-tuned as part of the main method. 
The dominant costs are repeated 4-bit GPTQ/AWQ quantization and deterministic option-scoring evaluation over saved datasets. 
We retained all prediction JSONL files, summary CSVs, merged reports, and derived-analysis tables, but cleaned quantized checkpoint directories and Hugging Face caches during experimentation. 
We report the hardware environment rather than a precise GPU-hour total because exploratory runs, sanity checks, and reruns were not logged with a separate compute ledger; the final experiments are reproducible on the single-GPU infrastructure described above.

\paragraph{Dataset sizes and splits.}
Table~\ref{tab:dataset_sizes} reports the deterministic evaluation sizes used after filtering and subsampling. Calibration examples are drawn from training or calibration pools only; evaluation examples are never used to construct \dpq calibration strings.

\begin{table}[t]
\centering
\small
\begin{tabular}{@{}llr@{}}
\toprule
Dataset & Suite & Eval. examples \\
\midrule
ARC-Challenge & Old-core & 299 \\
\squad answerability & Old-core & 1000 \\
TruthfulQA & Old-core & 817 \\
ARC-Easy & Extra-\mcqa & 570 \\
BoolQ & Extra-\mcqa & 1000 \\
PIQA & Extra-\mcqa & 1000 \\
HellaSwag & Extra-\mcqa & 1000 \\
OpenBookQA & Extra-\mcqa & 500 \\
CommonsenseQA & Extra-\mcqa & 1000 \\
\bottomrule
\end{tabular}
\caption{Evaluation sizes after deterministic filtering and subsampling.}
\label{tab:dataset_sizes}
\end{table}

\paragraph{Implementation details.}
The released scripts use teacher-forced option scoring. For each option, the option text is appended to the prompt; only option tokens are scored, and the option score is the mean token log-probability unless explicitly disabled. All evaluation runs use left padding, maximum sequence length 2048, and deterministic JSONL prediction files containing scores, probabilities, predicted option, gold option, confidence, entropy, and an unknown flag for \squad answerability.

\paragraph{GPTQ and BNB settings.}
GPTQ runs use 4-bit quantization through the Transformers \texttt{GPTQConfig} interface with model sequence length 512, batch size 1, and at most 128 calibration examples unless the size ablation specifies 64 or 256. The \texttt{desc\_act} flag is off by default. The BNB baseline uses 4-bit NF4 with double quantization and the model dtype as compute dtype (bfloat16 if supported, else float16).

\paragraph{DPQ calibration construction.}
Candidate boundary pools are built from ARC-Challenge training examples and \squad training examples, with 512 candidates from each source before model-specific scoring; \squad candidates are balanced between answerable and unanswerable cases. Formal calibration files are built with seed 4242, and the candidate-pool construction uses seed 2027. DPQ ranks candidates by margin/confidence/entropy/answerability/NLL depending on the variant. For boundary strings, the selected prompt is concatenated with the \fp top options. Mixed recipes allocate a fraction of the calibration set to boundary/high-doubt strings and the remainder to WikiText/random-QA anchors. This formatting is kept fixed across \dpq variants, so the reported ratio, size, component, and negative-control ablations compare selection choices under a shared boundary-string construction; fully separating formatting effects from selection effects is a natural follow-up control.

\paragraph{Prompt format and option scoring.}
Multiple-choice prompts are formatted as a question followed by labeled choices and the string ``The correct answer is''. Candidate options are scored as label continuations. For \squad answerability, each prompt contains the passage, question, and two options: ``Answerable from the passage'' and ``Unanswerable from the passage''. Scores are length-normalized by default and converted to option probabilities with a softmax.

\paragraph{Post-hoc controls.}
Temperature and score-space post-hoc methods are evaluated from saved logits/probabilities. Temperature scaling is fitted per evaluation setting for the additional temperature check. Flexible score-space calibrators, including option-bias, vector, matrix, Dirichlet, and isotonic mappings, are evaluated on saved option scores.

\paragraph{Artifacts, licenses, and retained outputs.}
We use public model checkpoints, quantization implementations, and standard NLP benchmarks for research evaluation only, without redistributing the original checkpoints or datasets. We retained prediction JSONL files, summary CSVs, merged reports, and derived-analysis tables; quantized checkpoint directories and local caches were cleaned during experimentation.

\end{document}